\documentclass{article}

    \usepackage[preprint]{neurips_2020}

\usepackage[utf8]{inputenc} 
\usepackage[T1]{fontenc}    
\usepackage{hyperref}       
\usepackage{url}            
\usepackage{booktabs}       
\usepackage{amsfonts}       
\usepackage{nicefrac}       
\usepackage{microtype}      
\usepackage{cite}
\usepackage{enumitem}
\usepackage{wrapfig}
\newtheorem{definition}{Definition}
\newtheorem{assumption}{Assumption}
\newtheorem{proof}{Proof}
\newtheorem{lemma}{Lemma}
\newtheorem{corollary}{Corollary}
\newtheorem{theorem}{Theorem}

\usepackage{amsmath}
\usepackage{amssymb}
\usepackage{graphicx}
\usepackage{subfigure}
\usepackage{caption}
\usepackage{natbib}
\usepackage{color}
\newcommand{\LB}{LB^{u}}
\newcommand{\UB}{UB^{u}}

\title{Inject Machine Learning into Significance Test for Misspecified Linear Models}

\author{%
Jiaye Teng\\
Shanghai University of Finance and Economics \\
\texttt{2016110299@live.sufe.edu.cn}\\
\And
Yang Yuan *\\
Tsinghua University\\
\texttt{yuanyang@tsinghua.edu.cn}\\
}

\begin{document}

\maketitle

\begin{abstract}
Due to its strong interpretability, linear regression is widely used in social science, from which significance test provides the significance level of models or coefficients in the traditional statistical inference.
However, linear regression methods rely on the linear assumptions of the ground truth function, which do not necessarily hold in practice. 
As a result, even for simple non-linear cases, linear regression may fail to report the correct significance level. 

In this paper, we present a simple and effective assumption-free method for linear approximation in both linear and non-linear scenarios. 
First, we apply a machine learning method to fit the ground truth function on the training set and calculate its linear approximation.
Afterward, we get the estimator by adding adjustments based on the validation set.
We prove the concentration inequalities and asymptotic properties of our estimator, which leads to the corresponding significance test. 
Experimental results show that our estimator significantly outperforms linear regression for non-linear ground truth functions, indicating that our estimator might be a better tool for the significance test. 
\end{abstract}

\section{Introduction}

Linear regression is commonly used in practice because it can provide explanations for important features by significance test, but the required linear assumptions (e.g. linearity and normality) on the ground truth function are easily violated in practice \citep{osborne2002four,casson2014understanding}. 
This problem is critical: 
when linear assumptions are violated, 
one may ignore important features or focus on unimportant features due to the \emph{fake significance level}. 
One potential solution is considering all kinds of ground truth functions by removing the linear assumption, but this triggers another problem: since there are numerous types of non-linear functions, how can we learn the feature importance without knowing the exact type of the ground truth function?

The answer is 
simply applying linear regression to the unknown non-linear ground truth functions, i.e. using \emph{misspecified linear models}~\citep{fahrmexr1990maximum,hainmueller2014kernel,grunwald2017inconsistency,markiewicz2019linear}. 
Apart from the fake significance level, 
misspecified linear models can also address other problems of linear regression, such as bias, inefficiency, and incorrect inferences (e.g. \citet{king2006the}). Indeed, 
as we will show in Section~\ref{Sec: Experiment}, traditional significance test based on linear regression fails even for the simple non-linear ground truth functions like square function.

The most common approach for misspecification is introducing high-order terms and interactions (e.g. \citet{friedrich1982defense,brambor2006understanding}), but this only works for the prescribed types and usually cannot find the correct functional form.
Another line of work \citep{white1980using,berk2013valid,mackinnon1985some,buja2015models,bachoc2020uniformly} tries to do the significance test directly based on the least square estimation, and derive the consistent estimators of its variance. 
The downside of this approach is that the corresponding estimators contain inevitable system errors and bias due to wrong model selections (see Section~\ref{Sec: Unbiaedness}).

In this work, we introduce machine learning methods into the misspecified linear models, where we do not need to know the correct functional form and also effectively avoid system errors.
We first use a machine learning method to fit the ground truth function in the training step and estimate the corresponding linear approximation.
Afterward, we correct the mistakes made by the machine learning methods in the validation step.
We show a positive correlation between the performance of the underlying machine learning method and the performance of our new estimator (see Theorem~\ref{Thm: MSE}).
Moreover, we prove the concentration inequalities (see Theorem~\ref{Thm: fillingTheBias}) and asymptotic properties (see Theorem~\ref{Thm: Asymptotic}) of the newly proposed estimator, which can be further applied into the significance test.

Several experiments are conducted to show that this newly proposed estimator works well in both non-linear and linear scenarios. 
Especially, our newly proposed estimator can significantly outperform the traditional linear regression (see Table~\ref{Tab: NonLinearCorrectness}) when considering the Kolmogorov-Smirnov statistic in the non-linear scenarios (square function).
This indicates that we make fewer mistakes in the significance test. 
For example, as we will show in Section~\ref{Sec: Exp_Nonlinear}, 
in the non-linear scenario, 
our method makes mistakes with probability $2.15\%$, while for the traditional linear regression, the number is $7.63\%$.

\section{Related Work}
The research on misspecified linear models can be broadly divided into \emph{Conformal Prediction}, which focuses on the inference of prediction, and \emph{Parameter Inference}, which focuses on the inference of the linear approximation parameter of ground truth function.

\textbf{Conformal Prediction}
is a framework pioneered by \citet{DBLP:journals/sigact/Law06}, which uses past experience to determine precise levels of confidence in new predictions. Conformal prediction \citep{DBLP:journals/jmlr/ShaferV08,DBLP:journals/corr/PapadopoulosVG14,barber2019conformal,2020arXiv200410181C,zeni2020conformal} mainly focuses on the confidence interval for predictions, so it cannot provide explanations for feature importance.
Our work can be regarded as a parallel line of conformal prediction that focuses on the assumption-free parameter estimation confidence interval.

\textbf{Parameter Inference} can be dated back to \citet{white1980using,mackinnon1985some}, where sandwich-type estimators for variance are proposed. 
Furthermore, \citet{buja2015models} introduce M-of-N Bootstrap techniques to improve the estimation of variance. 
\citet{hainmueller2014kernel} reduce this misspecification bias from a kernel-based point of view.
Some other techniques, e.g. LASSO \citep{lee2016exact}, least angle regression \citep{taylor2014exact} are introduced in the post-selection inference.
And some works \citep{rinaldo2016bootstrapping,buhlmann2015high} focus on a high-dimensonal reversions.
More discussions can be found in \citet{berk2013valid,bachoc2020uniformly}.
However, this line of works relies on the direct misspecification of linear models, which means system errors are inevitable when the ground-truth function is non-linear.
Furthermore, this type of estimator based on the least square estimation contains much bias, which will be further discussed in Section~\ref{Sec: Unbiaedness}.
In this paper, we use a machine learning based estimator instead of least square estimation, which contains less bias as we will see later in Section~\ref{Sec: Unbiaedness}.

\section{Preliminaries}
In this section, we define the basic notations, starting from the definition of function norm and function distance.
\begin{definition}[Function Norm and Function Distance]
\label{Def: FuncNorm}
Given a functional family $\mathcal{H}$ defined on domain $\mathbb{D}$, 
for any  $f \in \mathcal{H}$ and
a probability distribution $\mathbb{P}$
on $\mathbb{D}$ with density $\mu(x)$, 
the function $2$-norm of $f$ with respect to 
$\mathbb{P}$ is defined as 
\begin{equation*}
    \|f\|_{\mathbb{P}}\triangleq
\left(\int_{\mathbb{D}}[f(x)]^{2}  \mathrm{d} \mathbb{P} \right)^{1 / 2}=
\left(\int_{\mathbb{D}}[f(x)]^{2} \mu(x) \mathrm{d} x \right)^{1 / 2}
\end{equation*}

When the context is clear, we simply use $\|f\|$ instead. 
Moreover, the function distance between $f \in \mathcal{H}$ and $g \in \mathcal{H}$ is defined as 
\begin{equation*}
    d(f,g) \triangleq \|f-g\|=\left(\int_{\mathbb{D}}[f(x)-g(x)]^{2} \mathrm{d} \mathbb{P}\right)^{1 / 2}
\end{equation*}
\end{definition}

Based on function distance, we can define the least square linear approximation, or simply  linear approximation.

\begin{definition}[Linear Approximation]
For a given function $f(x)$, its least square linear approximation $g(x)$ is defined as
$$
g(x) \triangleq \operatorname{argmin}_{g \in \mathcal{G}}\left\|g-f\right\|
$$
where $\mathcal{G}$ is the linear functional family.
\end{definition}

The traditional linear regression uses a single dataset to compute the parameters, but our method splits the dataset into two parts, a training set, and a validation set, as defined below. 
\begin{definition}[Training and Validation Set]
\label{Def: Split}
Given a dataset $S=\left\{\left(x^{(i)}, y^{(i)}\right)\right\}_{i=1}^{N}$, we randomly split $S$ into training set $S_1$ and validation set $S_2$, where $S_1 \cap S_2 = \emptyset$, $S_1 \cup S_2 = S$, and $|S_1| = N_1$, $|S_2|=N_2$.
\end{definition}

In some scenarios, 
we may have additional $N_u$ unlabeled data points, therefore in total $N_t=N+N_u$ data points. 
Usually, unlabeled data are easier to get than labeled data, 
which can be used to calculating the linear approximation of the machine learning model, and help to improve the estimation of $\mathbb{E}(xx^T)$.
As we will discuss in Section~\ref{Sec: Experiment},
our analysis still applies without unlabeled data, given that the machine learning model is the linear form, and $\mathbb{E}(xx^T)$ is estimated from only the validation set. 
But more unlabeled data could help enrich the patterns of our choices for machine learning models.

In order to evaluate the performance of our model on the population distribution, we need to estimate the upper and lower bounds of a given function $u$ (will be defined later). 
\begin{definition}[Upper and Lower Bounds]
\label{Def: EmpiricalBound}
Given a function $u(x,y)$ defined on $\mathbb{Z}$, the upper and lower bounds of $u$ is
\begin{align*}
&L B^{u}=\inf_{(x,y)\sim \mathbb{Z}} \left\{u\left(x,y\right)\right\}\\
&U B^{u}=\sup_{(x,y)\sim \mathbb{Z}} \left\{u\left(x,y\right)\right\}
\end{align*}
Similarly, given a dataset $S=\left\{\left(x^{(i)}, y^{(i)}\right)\right\}_{i=1}^{N}$, where $\left(x^{(i)}, y^{(i)}\right) \sim \mathbb{Z}$, 
the empirical lower bound and empirical upper bound of $u\left(x, y\right)$ on set $S$ is defined by
$$\begin{aligned}
&L B_{S}^{u}=\inf _{\mathrm{i}}\left\{u\left(x^{(i)}, y^{(i)}\right)\right\}\\
&U B_{S}^{u}=\sup _{\mathrm{i}}\left\{u\left(x^{(i)}, y^{(i)}\right)\right\}
\end{aligned}$$
\end{definition}

While $\LB$ and $\UB$ are hard to get, we may assume that $u$ is at least loosely bounded. 
\begin{assumption}
\label{assump:bounded}
Given a loss function $u(x,y)$ defined on $\mathbb{Z}$, we assume $u\in [0,1]$.
\end{assumption}
Notice that for linear regression, Assumption~\ref{assump:bounded} usually holds, as we may assume that the domain of input $x$ is bounded, and the weight $w$ is also bounded. After applying proper scaling, we get $u\in [0,1]$.

As we will see in Theorem~\ref{Thm: MSE}, our analysis depends on  $R\triangleq\UB-\LB$, and smaller $R$ gives more accurate estimations. 
If Assumption~\ref{assump:bounded} holds, we immediately have $R\leq 1$. 
However, with additional prior knowledge on $\LB$ and $\UB$, we may get tighter bounds of $R$  using Bayesian methods, as discussed in Lemma~\ref{Lem: BoundForLoss}.



\begin{lemma}[Tighter Estimation on $R$]
\label{Lem: BoundForLoss}
Given a validation dataset $S = \{(x_i, y_i)\}_{i=1}^{N}$ where data points are randomly sampled from $\mathbb{Z}$, and a function $u$ defined on $\mathbb{Z}$ with bounds $\LB$ and $\UB$ where $\LB$ and $\UB$ are unknown, and $u \sim U(\LB, \UB)$. Let $R=\UB-\LB$, and assume the prior $R \sim U(0,1)$, where $U$ represents uniform distribution. For any $\delta_0 \in (0,1)$,
if $N \geq 2$, we have
$$ \mathbb{P}_{S}\left(UB^u - LB^u \leq (UB^u_{S} - LB^u_{S}) \delta_0^{-\frac{1}{N-1}}\right) \geq 1-\delta_0
$$
where $LB_{S}^u, UB_{S}^u$ are the empirical bounds of $u$ on set $S$. 
\end{lemma}

The following Assumption~\ref{assump: Concentration} and Assumption~\ref{assump: Bounded Covariance} mainly focus on the explanatory variables $x$.

\begin{assumption}[Concentration of Explanatory Variables]
\label{assump: Concentration}
Let $x$ be a random vector in $\mathbb{R}^d$, and assume that there exists a constant $K\geq 1$, such that $$\Vert x \Vert_2 \leq K(\mathbb{E}(\Vert x \Vert_2^2))^{\frac{1}{2}}$$ almost surely. 
\end{assumption}

\begin{assumption}[Bounded Second Moment]
\label{assump: Bounded Covariance}
Let $x$ be a random vector in $\mathbb{R}^d$, and assume $\mathbb{E}(xx^T)$ is invertible.
Denote $A \triangleq [\mathbb{E}(xx^T)]^{-1} $ and $\hat{A}_n \triangleq [\frac{1}{n} \sum_{i=1}^{n} x_i x_i^T]^{-1}$.
We assume that 
$$m I_{d} \preccurlyeq A, \hat{A}_n \preccurlyeq M I_{d}$$
\end{assumption}

In the following statements, we always use all data, including labeled and unlabeled data, to estimate $A$, leading to the estimator $\hat{A}_{N_t}$.

\section{Estimation}
In this section, we study the problem of linear approximation $g^*$ of the oracle model $f^*(x)=y$, based on a machine learning framework. 
Specifically, we show the following: 
a) how the performance of our machine learning model affects the linear approximation estimator. 
b) after adding a bias term (we call it the residual term), one can get an estimator with better guarantees.
c) how to run the hypothesis test (coefficient significance) based on the asymptotic distribution of our estimator.
We defer all the proofs to Appendix~\ref{appendix: proof}.

\subsection{Approach Based on MSE}

In this subsection, we study the relationship between the linear approximation functions $g^1$ and $g^*$ given that $f^1$ is close to $f^*$. 
We use mean squared error (MSE) to measure the performance of machine learning models. 

\begin{theorem}[Performance of Machine Learning Models]
\label{Thm: MSE}
For $x \in \mathbb{R}^d$, given oracle model $f^*$ and machine learning model $f^1$, where their linear approximations are $g^{*}= w^{*}x $ and $g^{1}= w^{1}x$, respectively.  
The labeled data is randomly split following Definition~\ref{Def: Split}. 
Denote the loss function as $l (x, y) \triangleq \left(f^{1}(x)-y\right)^{2}$, 
the population loss is then $L = \int l(x,y) \mathrm{d} \mathbb{Z}$. 
Under Assumption~\ref{assump: Bounded Covariance},
for any $\delta \in(0,1)$, the following inequality holds $\text{w.p.} \geq 1-\delta$:
\begin{equation*}
\begin{split}
\left\|g^{*}-g^{1}\right\|_{2} &\leq \sqrt{d} \frac{M}{m} \epsilon\\
\left\|w^{*}-w^{1}\right\|_{2} &\leq \sqrt{d} \frac{M}{\sqrt{m}} \epsilon 
\end{split}
\end{equation*}
where $\epsilon = \sqrt{\frac{1}{N_{2}} \sum_{i=1}^{N_{2}} l^{(i)}+(U B^{l}-L B^{l}) \sqrt{\frac{\log \left(\frac{1}{\delta}\right)}{2 N_{2}}}}$, $l^{(i)}=l(x^{(i)}, y^{(i)})$ is the $i_{th}$ sample loss defined on validation set.
\end{theorem}

\textbf{Remark:} 
Theorem~\ref{Thm: MSE} shows that 
$\epsilon$ controls the approximation quality of $g^1$ and $w^1$, which depends on both the validation set size and the validation loss.
In other words, 
if the machine learning model $f^1$ generalizes well, we get good estimations of $g^*$ and $w^*$. 

The term $\epsilon$ depends on $R=UB^l-LB^l$, which is bounded by $1$ based on Assumption~\ref{assump:bounded}. Although this term shrinks as the validation set grows, below we show that it can be bounded more accurately using Lemma~\ref{Lem: BoundForLoss}. 

\begin{corollary}
\label{Cor: MSE}
Under the assumptions of Lemma~\ref{Lem: BoundForLoss} and Theorem~\ref{Thm: MSE}, 
by replacing $u$ in Lemma~\ref{Lem: BoundForLoss} by loss function $l$ and plugging it into Theorem~\ref{Thm: MSE},
we have $w.p. \geq 1 - \delta - \delta_0$
\begin{equation*}
\begin{split}
\left\|g^{*}-g^{1}\right\|_{2} &\leq \sqrt{d} \frac{M}{m} \epsilon_{S_2}\\
\left\|w^{*}-w^{1}\right\|_{2} &\leq \sqrt{d} \frac{M}{\sqrt{m}} \epsilon_{S_2} 
\end{split}
\end{equation*}

where $\epsilon_{S_2} = \sqrt{\frac{1}{N_{2}} \sum_{i=1}^{N_{2}} l^{(i)}+(U B_{S_2}^{l}-L B_{S_2}^{l})\delta_0^{-\frac{1}{N_2-1}} \sqrt{\frac{\log \left(\frac{1}{\delta}\right)}{2 N_{2}}}}$
\end{corollary}

Intuitively, 
using $g^1$ (the best linear approximation of $f^1$) to approximate $g^*$ seems to the optimal choice. However, as we will show below, this is not true, as $f^1$ may contain bias in the linear setting. 


\subsection{Filling the Bias}
\label{Sec: FillingtheBias}
In this subsection, we will jump out of the restrictions of MSE, and improve our estimation by adding a bias term. 
We first present Lemma~\ref{Lem: covariance} that focuses on the estimation of the second-moment matrix of explanatory variables $\mathbb{E}(x x^T)$. 

\begin{lemma}[The Second Moment Concentration]
\label{Lem: covariance}
Under Assumption~\ref{assump: Concentration} and Assumption~\ref{assump: Bounded Covariance}, 
for every integer n, the following inequality holds $w.p.\geq 1-\delta/2$:
$$
\Vert \hat{A}_n - A \Vert \leq C \frac{M^2}{m} \left( \sqrt{\frac{K^2 d \log{(4d/\delta)}}{n}} +\frac{K^2 d \log{(4d/\delta)}}{n} \right) 
$$
where $C$ is a constant, $A, \hat{A}_n$ are defined in Assumption~\ref{assump: Bounded Covariance}.
\end{lemma}

By adding a small bias, we can derive the following Theorem~\ref{Thm: fillingTheBias}, which mainly focuses on the coordinate wise bound. Here we denote $x_j$ as the $j_{th}$ feature of $x$, and $A^j$ as the $j_{th}$ row of matrix $A$.

\begin{theorem}[Adding a bias term]
\label{Thm: fillingTheBias}
For $x \in \mathbb{R}^d$, given oracle model $f^*$ and machine learning model $f^1$, where the linear approximation is $g^{*}= w^{*}x$ and $g^{1}= w^{1}x$, respectively.  
The labeled data is randomly split based on Definition~\ref{Def: Split}. 
Assume that  Assumption~\ref{assump: Concentration} and Assumption~\ref{assump: Bounded Covariance} hold.
Denote $z=x\left(f^{*}(x)-f^{1}(x)\right)$.
Then, for any $\delta \in(0,1)$, the following inequality holds $\text{w.p.} \geq 1-\delta$:
\begin{equation*}
\begin{split}
\left|w^{*}_{j}-w^{1}_{j}-\frac{1}{N_{2}} \sum_{i=1}^{N_2} \hat{A}^{j}_{N_t} z^{(i)}\right| \leq & (U B^{A^j z}-L B^{A^j z}) \sqrt{\frac{\log \frac{4}{\delta}}{2 N_{2}}} \\
& + C \frac{M^2}{m} \Vert \frac{1}{N_2} \sum_{i=1}^{N_2} z^{(i)}\Vert  \left( \sqrt{\frac{K^2 d \log{(4d/\delta)}}{N_t}} +\frac{K^2 d \log{(4d/\delta)}}{N_t} \right)
\end{split}    
\end{equation*}
where $z^{(i)}$ is the realization of $z$ defined on validation set.
\end{theorem}

Notice that $N_t = N_1 + N_2 +N_t > N_2$ is the total number of data. 
The first term in the bound is because we use samples in the validation set to estimate the population.
A tighter bound requires smaller fluctuation for $A^j z$ (which is $U B^{A^j z}-L B^{A^j z}$) and more samples in validation set ($N_2$).
The second term is because we use $\hat{A}_{N_t}$ to replace $A$. 
A smaller condition number $\kappa = M/m$ and a larger $N_t$ help tighten the bound.

Similar to Corollary~\ref{Cor: MSE}, we can use $LB_{S_2}^{A^jz}$ and $UB_{S_2}^{A^jz}$ to replace $LB^{A^jz}$ and $UB^{A^jz}$ under additional assumptions, see below.

\begin{corollary}
\label{Cor: fillingTheBias}
Under the assumptions of Lemma~\ref{Lem: BoundForLoss} and Theorem~\ref{Thm: fillingTheBias}, 
by replacing $u$ in Lemma~\ref{Lem: BoundForLoss} by $A^jz$ and plugging it into Theorem~\ref{Thm: fillingTheBias}, 
we have $w.p. \geq 1 - 3\delta - \delta_0$
\begin{equation*}
\left|w^{*}_{j}-w^{1}_{j}-\frac{1}{N_{2}} \sum_{i=1}^{N_2} \hat{A}^{j}_{N_t} z^{(i)}\right| \leq (U B_{S_2}^{\hat{A}^{j}_{N_t} z}-L B_{S_2}^{\hat{A}^{j}_{N_t} z} + 2b) \delta_0^{-\frac{1}{N_2-1}} \sqrt{\frac{\log \frac{4}{\delta}}{2 N_{2}}} + b 
\end{equation*}
where $b = C \frac{M^2}{m} \max_{S_2}\Vert z^{(i)}\Vert  \left( \sqrt{\frac{K^2 d \log{(4d/\delta)}}{N_t}} +\frac{K^2 d \log{(4d/\delta)}}{N_t} \right)$
\end{corollary}

Therefore, we should use $w^{e}_{j} = w^{1}_{j} + \frac{1}{N_{2}} \sum_{i=1}^{N_2} \hat{A}^{j}_{N_t} z^{(i)}$ as the new estimator. 
Recall the bound in Theorem~\ref{Thm: MSE} (denoted as $B_1$) and the bound in Theorem~\ref{Thm: fillingTheBias} (denoted as $B_2$). We can see that as $N_2$ goes to zero, $B_2 \to 0$ while $B_1 \to \sqrt{d \bar{l}}M/m$, where $\bar{l}$ denote the average sample loss. 
This means that $g^1$ cannot be arbitrarily similar to $g^*$ given a fixed machine learning model even if we have infinite data for validation. 
That is to say, although Theorem~\ref{Thm: MSE} contains the frequently-used MSE as a measure, it causes some natural bias. And Theorem~\ref{Thm: fillingTheBias} filled this bias by adding a correction term.


Furthermore, as the standard practice in statistics, we will derive the asymptotic property of $w^e$ in Section~\ref{Sec: AsymptoticProperties}. 

\subsection{Asymptotic Properties}
\label{Sec: AsymptoticProperties}

In this section, we study the asymptotic property of estimator $w^e$, which gives us tighter and more practical guarantees. 
In the following analysis, we assume that $f^* \neq f^1$. This is without loss of generality because otherwise we can directly use $w^1$ to estimate $w^*$ without any loss. 

\begin{theorem}[Asymptotic Properties]
\label{Thm: Asymptotic}
Given oracle model $f^*$ and machine learning model $f^1$, with its corresponding linear approximation $g^{*}= w^{*}x$, $g^{1}= w^{1}x $, respectively.  
The labeled data is randomly split based on Definition~\ref{Def: Split}. 
Denote $z=x\left(f^{*}(x)-f^{1}(x)\right)$, 
and assume $w^*, w^e, z$ are bounded\footnote{This usually holds in practice as long as $x, y, f^1$ are all bounded.}, 
then under Assumption~\ref{assump: Bounded Covariance} 
, the following asymptotic property of $w^e = w^1 + \frac{1}{N_2} \sum_{i=1}^{N_2} \hat{A}_{N_t} z^{(i)}$ holds:
\begin{equation*}
\hat{\Sigma}_e^{-1} \left(w^e-w^*\right) \stackrel{d}\sim \mathcal{N}{(0, I_d)}
\end{equation*}
where $\mathcal{N}()$ represents normal distribution, $\hat{\Sigma}_{e}^{2}=\frac{1}{N_{2}} \hat{A}_{N_t} \hat{D} \hat{A}_{N_t}$, $\hat{D} = \frac{1}{N_2-d} \sum_{i=1}^{N_2} (z^{(i)} - \bar{z}) (z^{(i)}- \bar{z})^T $, $\bar{z} = \frac{1}{N_2} \sum_{i=1}^{N_2} z^{(i)} $. 
\end{theorem}

\textbf{Remark:}
Traditional asymptotic analysis is usually based on the assumptions for the ground truth function, but our analysis does not need such assumptions and instead relies on the training-validation framework.  
For example, 
in the traditional analysis, 
the claim that $w^{LSE}$ follows normal distribution is based on the assumption that the ground truth function is linear, and also the label noise follows a well-defined distribution. 




Now we can do the hypothesis test based on the asymptotic property of $w^e$, including model test and coefficient test under significance level $\delta$.
The details can be found in Appendix~\ref{appendix:test}.

\section{Experiments}
\label{Sec: Experiment}
In this section, we conduct experiments of the significance test derived in Section~\ref{Sec: AsymptoticProperties}. 
We will show that our method works in both linear and non-linear scenarios, while the traditional linear regression fails in non-linear scenarios even in a simple square case.
Due to space limitations, we show linear scenarios in Appendix~\ref{Sec: Exp_linear}.
More experimental details could be found in Appendix~\ref{Sec: Experiment Details}.
For each statistic, we repeat experiments 6 times and compute its $95\%$ confidence interval of its mean.

We choose two types of machine learning models: a three-layer Neural Network (labeled as Ours(NN)) and a Linear-form model (labeled as Ours(L)), respectively. 
Note that Ours(L) does not need unlabeled data, while Ours(NN) needs unlabeled data to calculate the linear approximation of machine learning methods.

\subsection{Metrics}

Two metrics are considered here, focusing on the \emph{correctness} and \emph{efficiency}, respectively.

Correctness is shown by the Kolmogorov-Smirnov statistic, which is defined in Equation~\ref{Eqn: KS}. 
Kolmogorov-Smirnov statistic measures how close the simulation results and theoretical results are. 
Smaller Kolmogorov-Smirnov statistic is better.
\begin{equation}
\label{Eqn: KS}
    KS \triangleq \max_x |F(x) - G(x)|
\end{equation}
where $F$ is the empirical CDF in simulation, and $G$ is the theoretical CDF.

Efficiency is shown by the average standard deviation ($\bar{\sigma}$) of the estimators. 
Efficiency measures how much uncertainty the new estimators have. 
Since we have removed the requirements of linear assumptions, more uncertainty appears in our newly proposed method.
A smaller $\bar{\sigma}$ means that we have more confidence in our estimators.

\subsection{Unbiasedness}
\label{Sec: Unbiaedness}
First of all, we would show that the traditional estimators based on least squares (LSE) contain more bias (see also \citep{rinaldo2016bootstrapping}), including the linear regression methods and the estimators proposed in \citet{white1980using,mackinnon1985some,lee2016exact,taylor2014exact,buhlmann2015high,bachoc2020uniformly}, etc.

\begin{table}[htbp]
\vspace{-15pt}
	\centering 
	\caption{Bias Comparison (Non-Linear)} 
	\label{Tab: Unbiaseness} 
	\begin{tabular}{c|c|c|c}  
		\hline 
		\  & LSE-based Estimator & Ours(NN) & Ours(L)  \\
		\hline
		$w_0$ & { -0.0048 ($\pm 0.0013$)} & -0.0010 ($\pm 0.0036$) & -0.0001 ($\pm 0.0020$)\\
		$w_1$ &0.0015 ($\pm 0.0020$)& 0.0048 ($\pm 0.0110$)& 0.0021 ($\pm 0.0037$) \\
		\hline
	\end{tabular}
\vspace{-10pt}
\end{table}

\begin{wrapfigure}{r}{0.5\textwidth}
\vspace{-5pt}
    \begin{center}
    \includegraphics[width=0.5\textwidth]{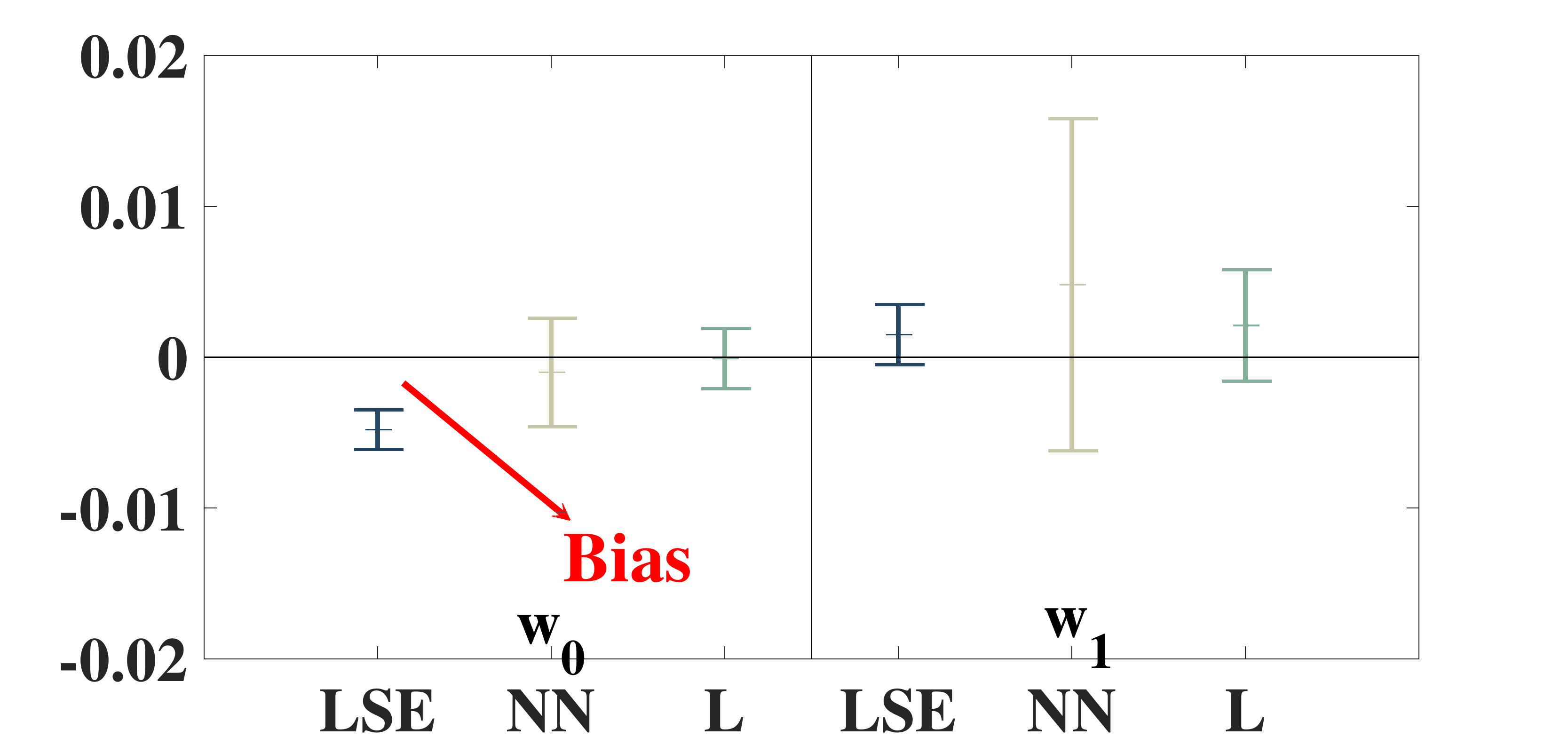}        
    \end{center}
    \caption{Confidence interval for each bias. LSE, NN, L is short for traditional LSE-based estimators, Ours(NN) and Ours(L), respectively.}
    \label{Fig: bias}
\vspace{-12pt}
\end{wrapfigure}

We test a simple square case, where $f^*(x) = x^2, x\sim U(0,1)$, our aim is to estimate its linear approximation.
We repeat the simulation 1000 times, and each time we calculate the mean of the estimators. 
Note that for better showing the bias, we use a smaller dataset (see Appendix~\ref{appendix: unbiased}).
Table~\ref{Tab: Unbiaseness} shows the difference between simulation results and the theoretical parameter with their $95\%$ confidence interval.

Figure~\ref{Fig: bias} shows that the traditional LSE-based estimators on $w_0$ are \emph{biased}, where the confidence interval of bias does not contain $0$.
Our proposed methods can outperform these LSE-based methods because the newly proposed methods have a smaller bias.
For simplification, we compare our proposed methods with only linear regression on their correctness and efficiency in Section~\ref{Sec: Exp_Nonlinear}, since linear regression is most widely used among these LSE-based methods in practice.

\subsection{Non-linear Scenarios}
\label{Sec: Exp_Nonlinear}
In this section, we focus on the performance of the linear regression method and our newly proposed method under a non-linear scenario. We focus on a simple non-linear ground truth model, which is
\begin{equation*}
f^*(x) = x^2, x\sim U(0,1)
\end{equation*}
with no randomness $\epsilon$. 
Its linear approach can be theoretically calculated by
\begin{equation*}
    g^*(x) = -\frac{1}{6}+x
\end{equation*}
Our aim is to estimate $w^* = (-\frac{1}{6}, 1)$. 
Thus the hypothesis test can be written as
\begin{equation*}
    H_0: w^*=(-\nicefrac{1}{6}, 1)\quad \text{v.s.} \quad H_1: w^* \not= (-\nicefrac{1}{6}, 1)
\end{equation*}

We repeat the simulation 1000 times, each time we calculate the statistic and plot them in Figure~\ref{Fig: test_x_2}. 
We also plot its theoretical distribution, which helps visualize how far the simulation results and theoretical results are.
It is visualized in Figure~\ref{Fig: test_x_2} that traditional linear regression fails in \emph{even} a simple square case, while our new estimator works well. We choose one group of the six here to show the figure.

\begin{figure}
\vspace{-50pt}
    \centering
\vspace{-10pt}
    \subfigure[Baseline: $w_0$]{
        \includegraphics[width=0.3\textwidth]{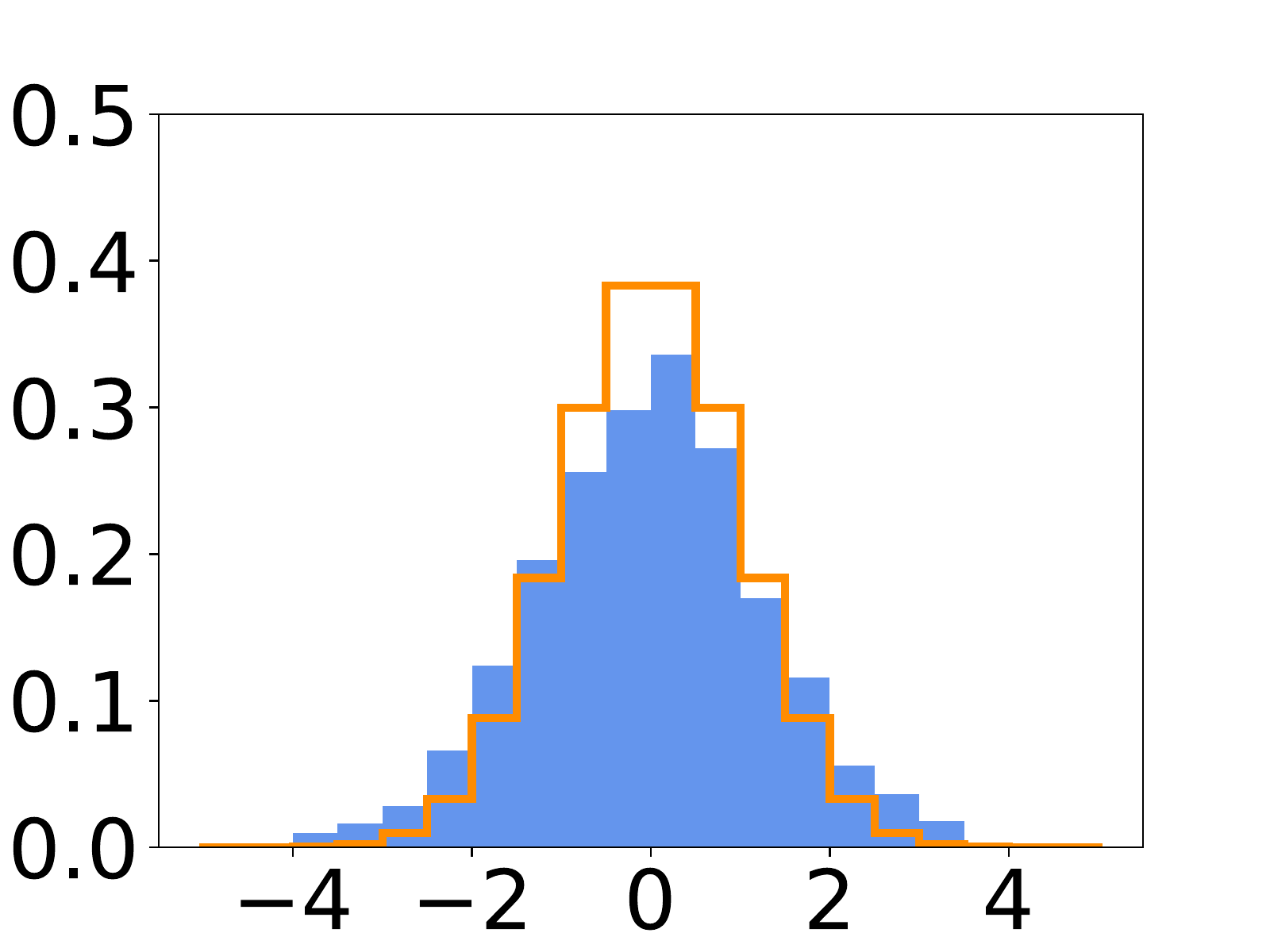}
    }
    \subfigure[Baseline: $w_1$]{
	\includegraphics[width=0.3\textwidth]{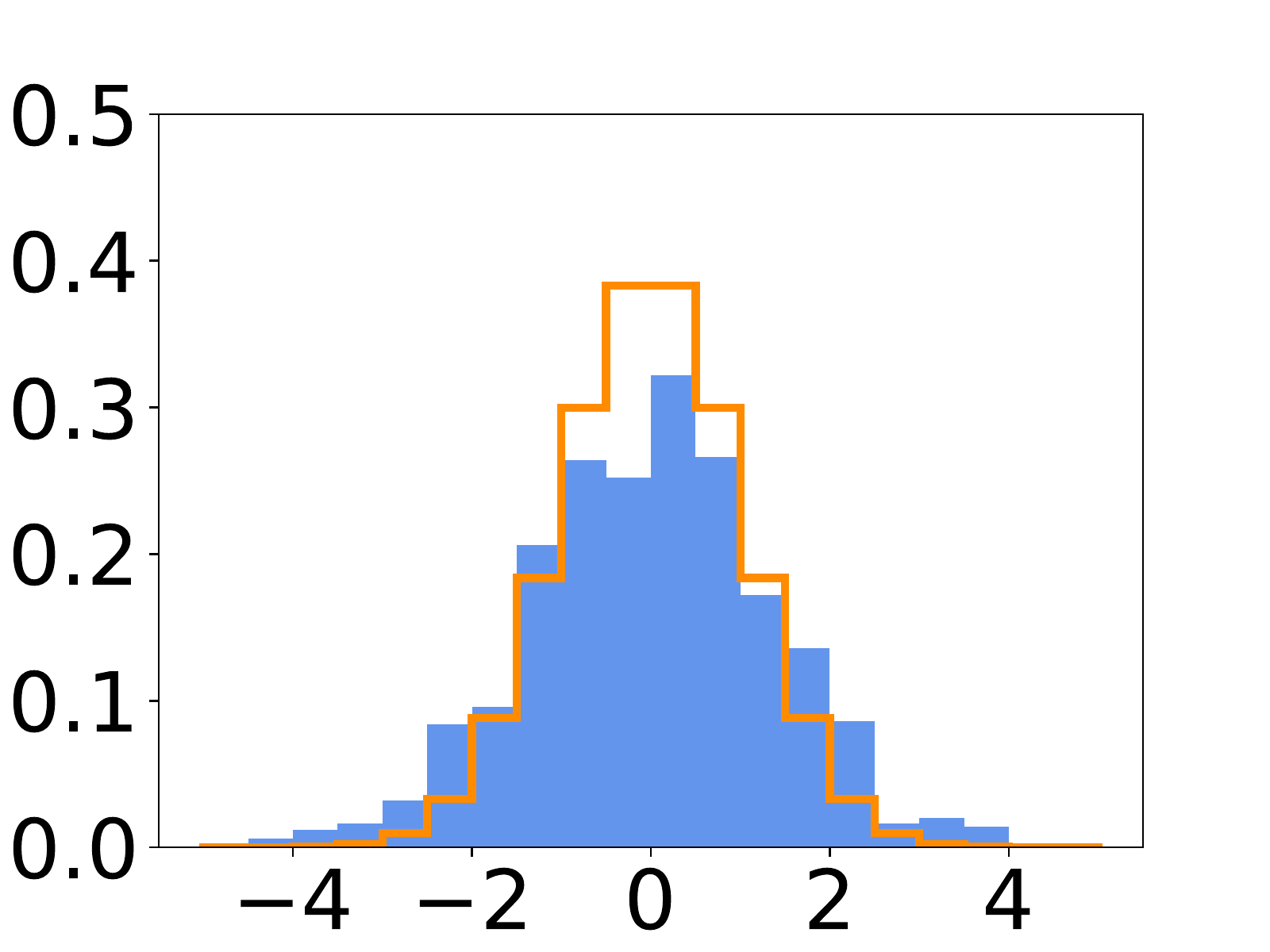}
}
    \subfigure[Baseline: $w$]{
	\includegraphics[width=0.3\textwidth]{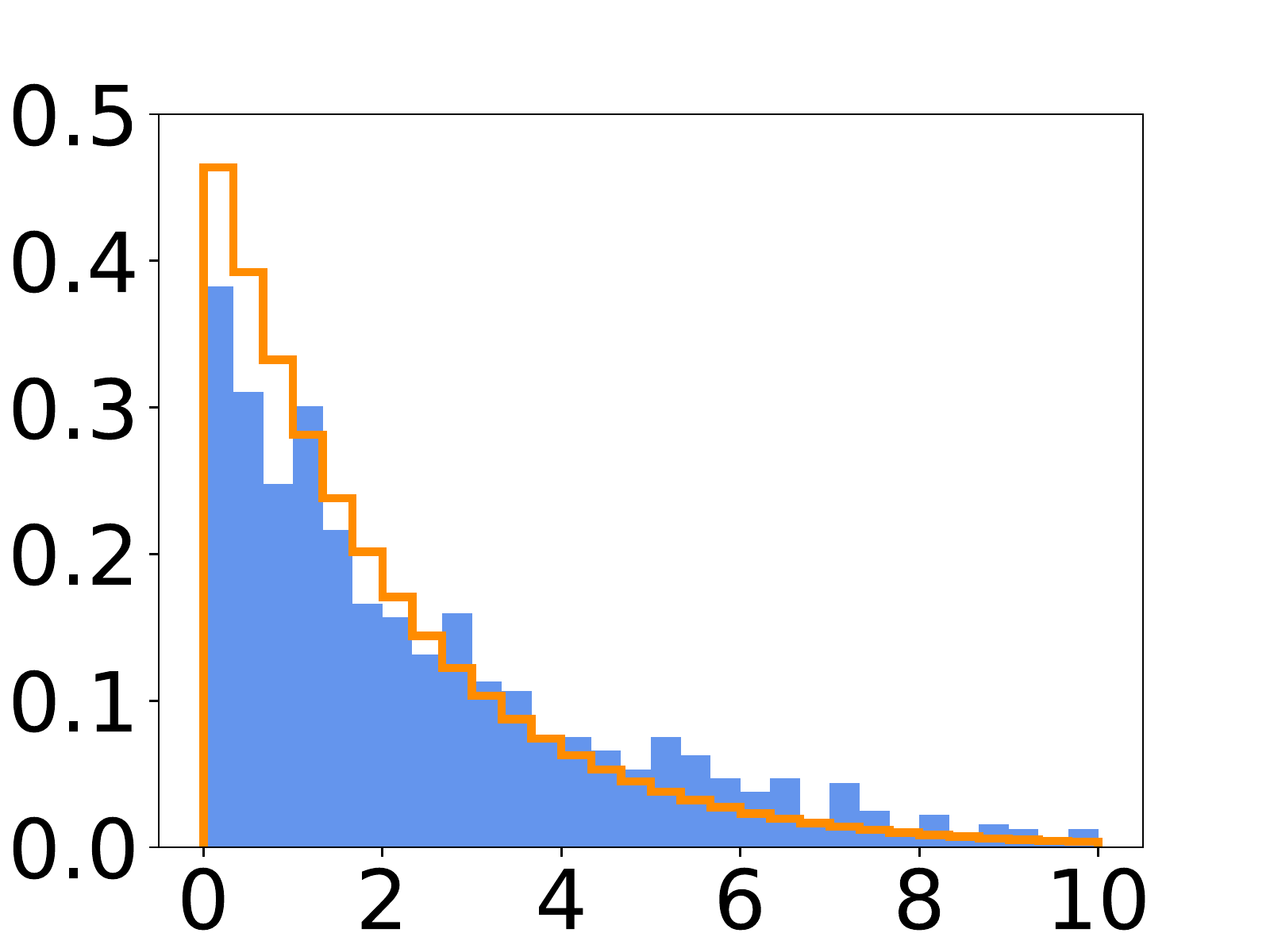}
}

\vspace{-11pt}
    \subfigure[Ours(NN): $w_0$]{
        \includegraphics[width=0.3\textwidth]{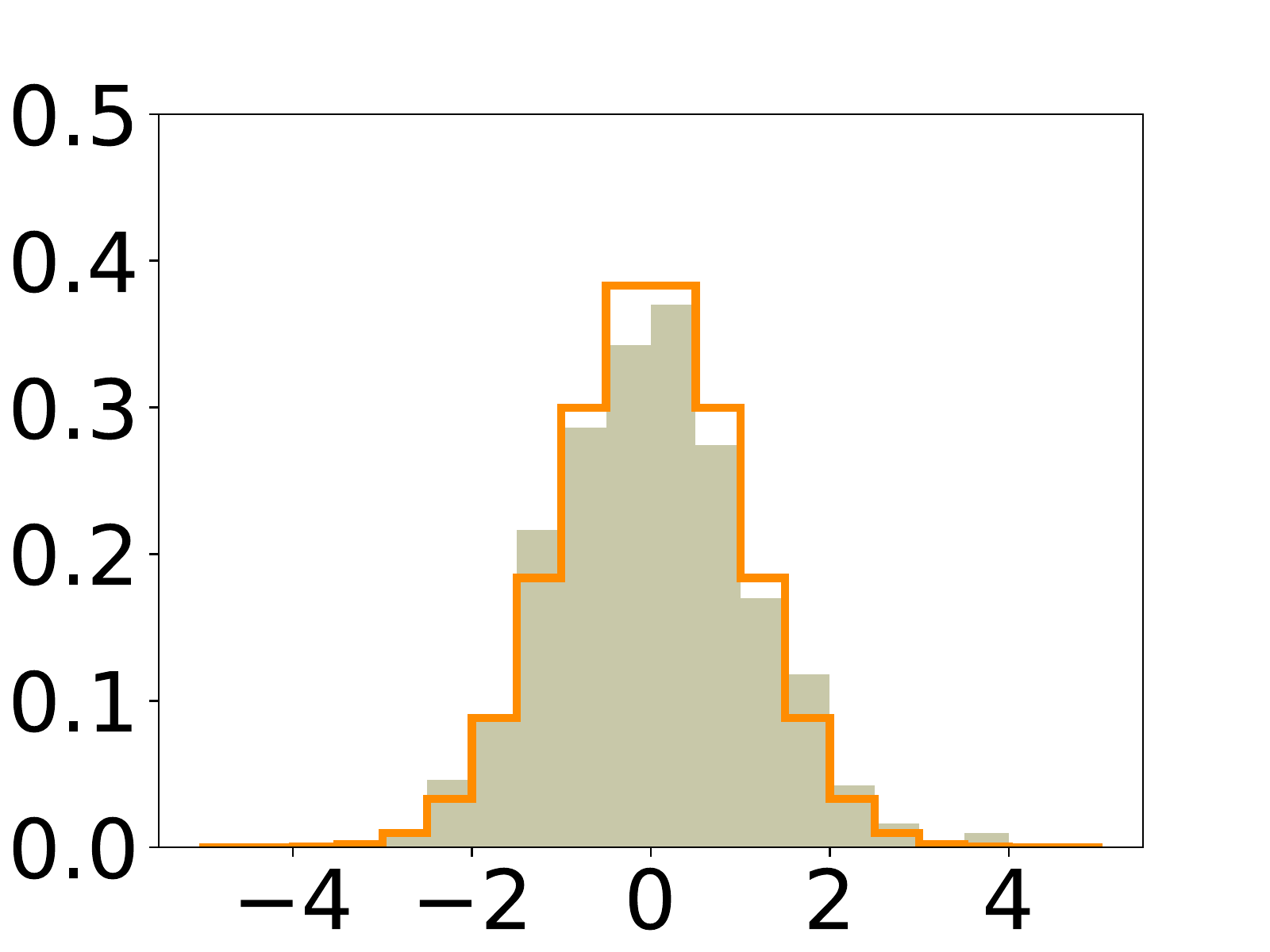}
    }
    \subfigure[Ours(NN): $w_1$]{
	\includegraphics[width=0.3\textwidth]{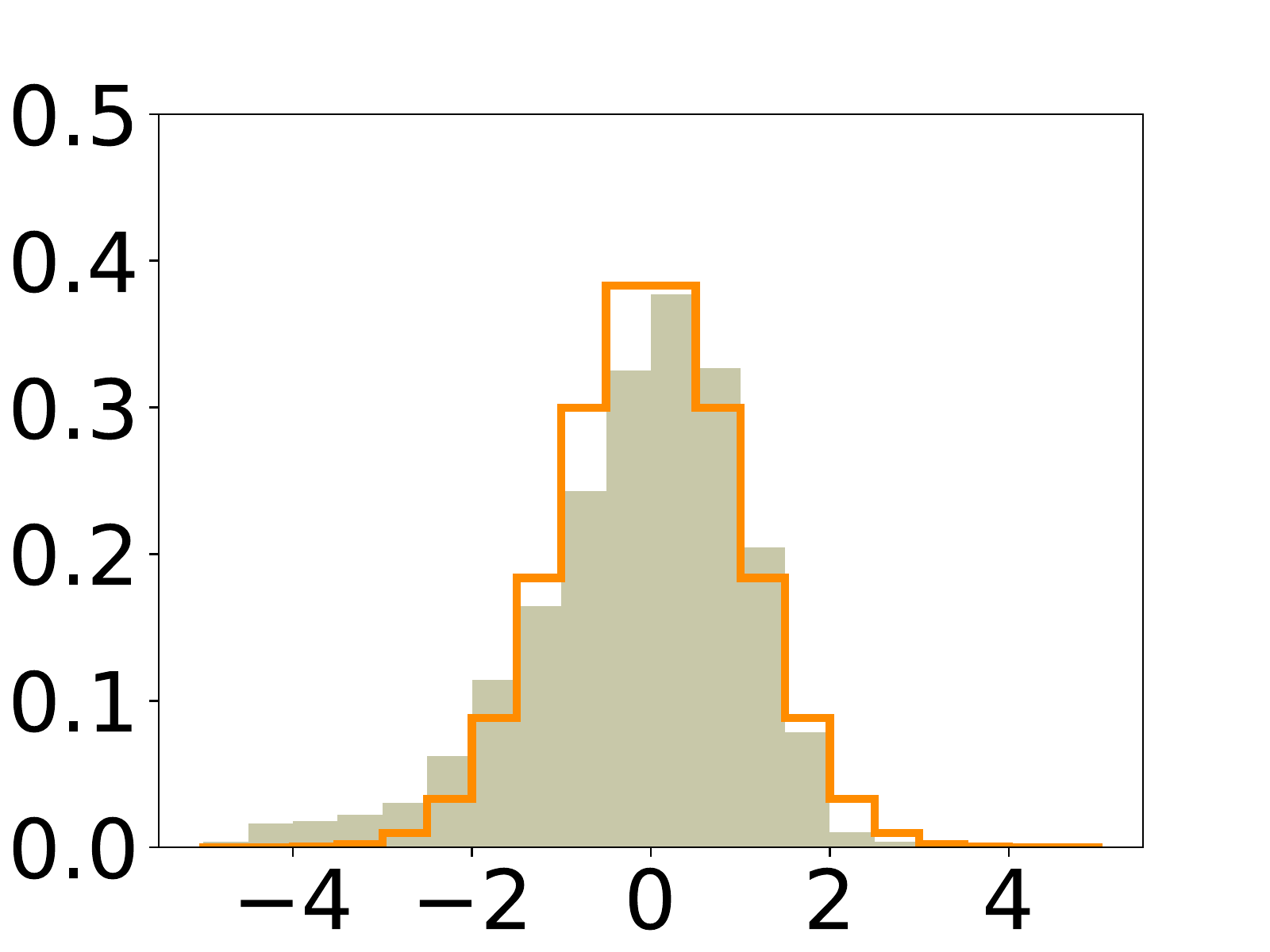}
}
    \subfigure[Ours(NN): $w$]{
	\includegraphics[width=0.3\textwidth]{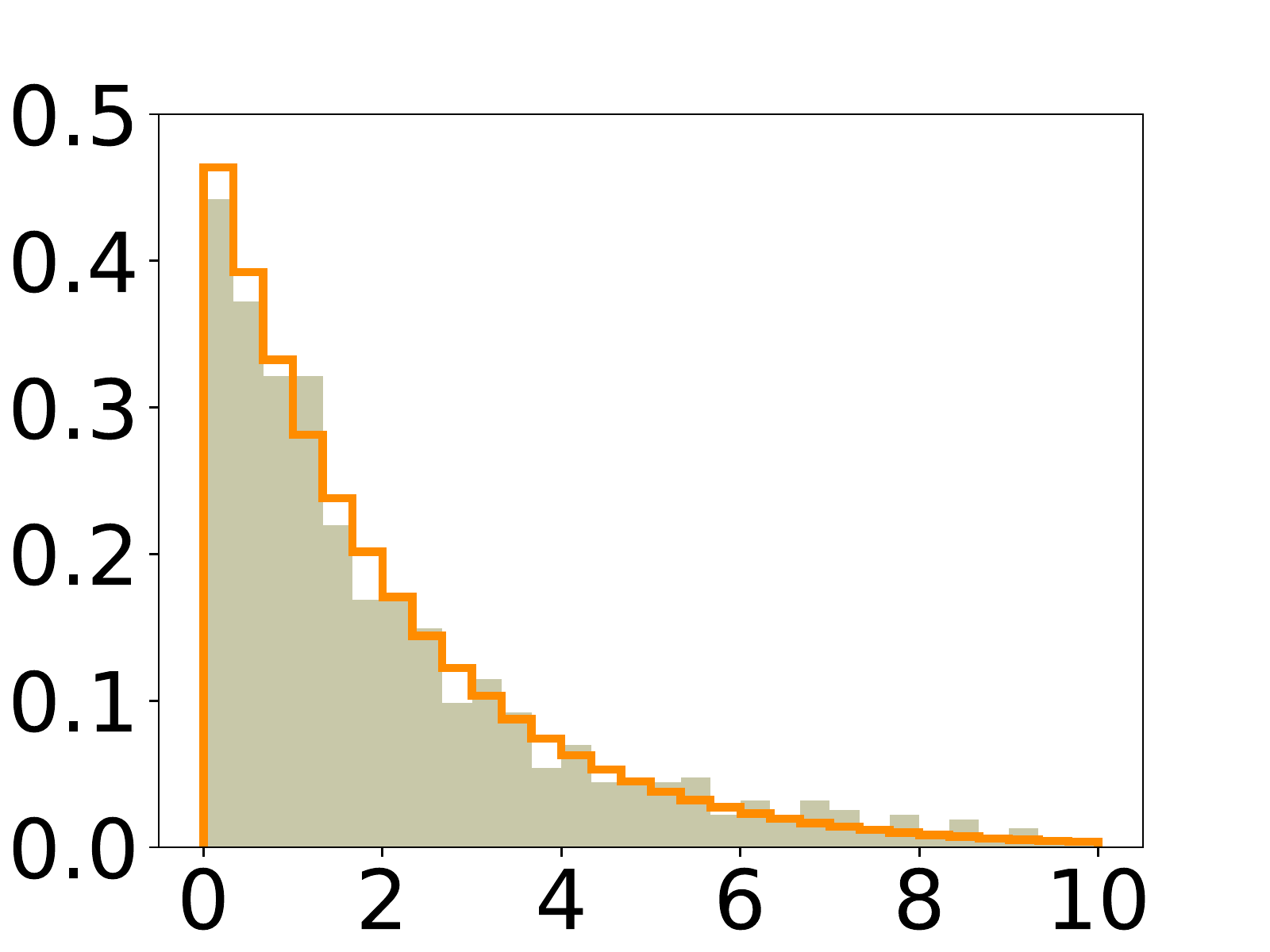}
}

\vspace{-11pt}
    \subfigure[Ours(L): $w_0$]{
        \includegraphics[width=0.3\textwidth]{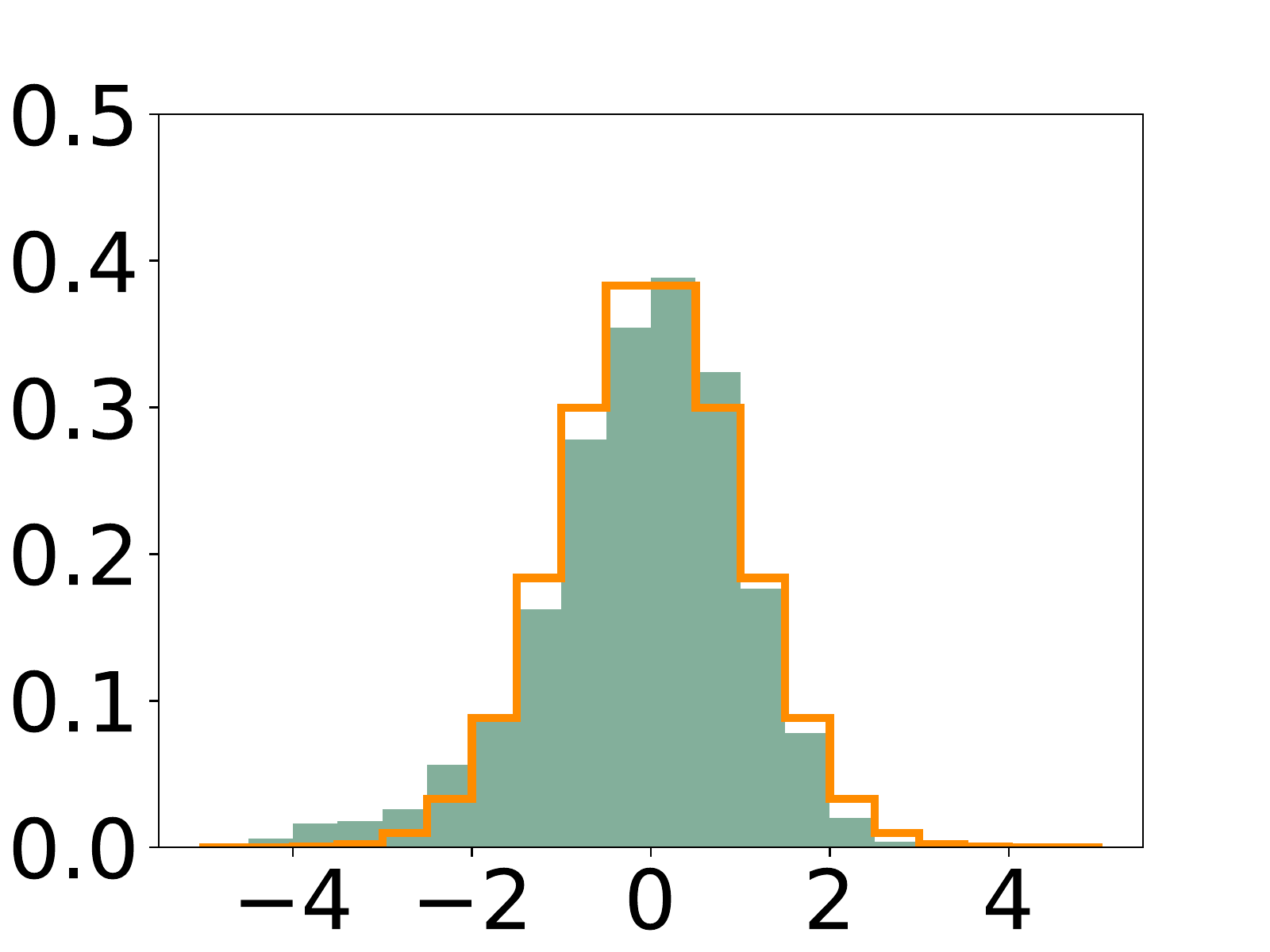}
    }
    \subfigure[Ours(L): $w_1$]{
	\includegraphics[width=0.3\textwidth]{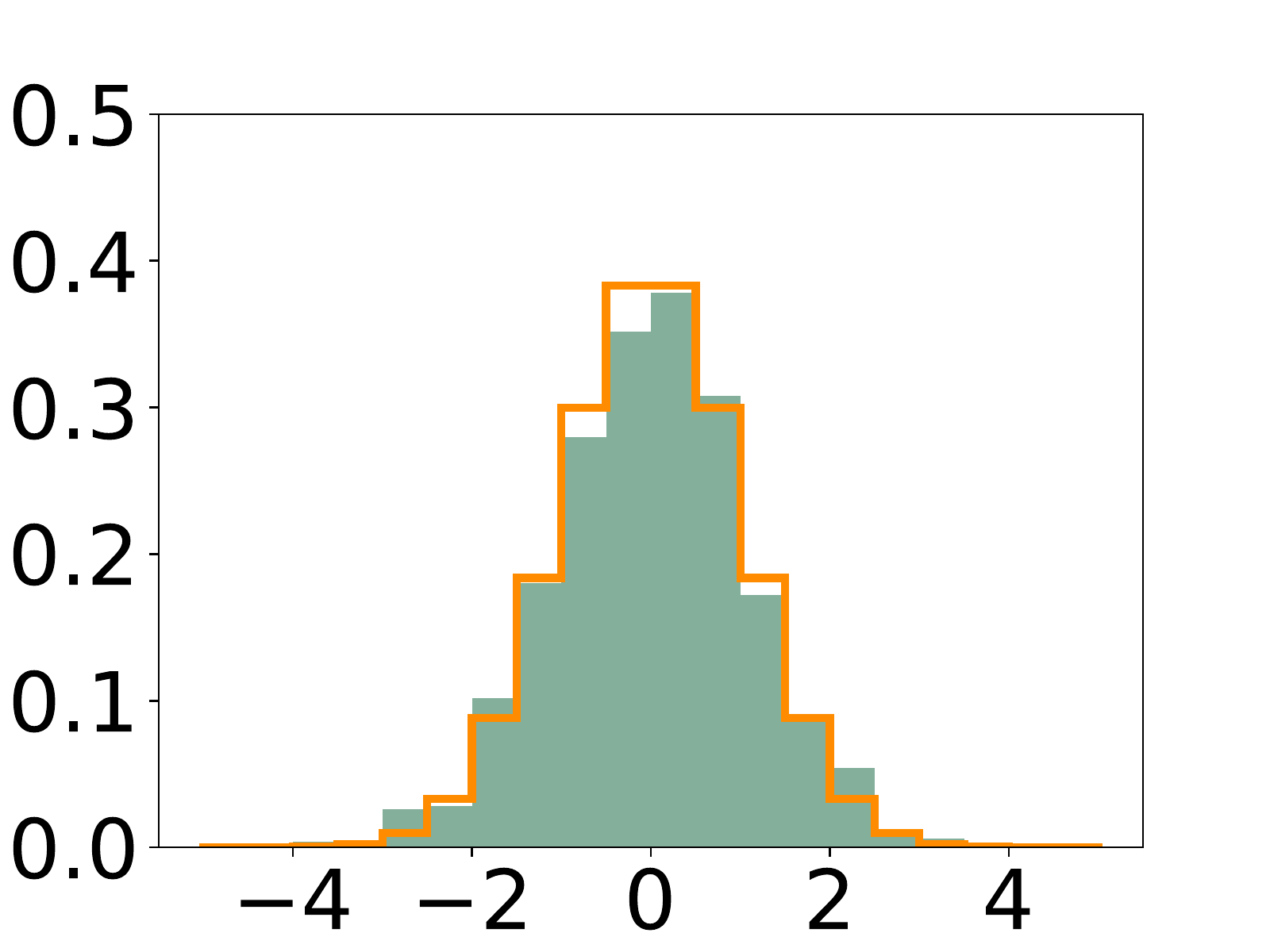}
}
    \subfigure[Ours(L): $w$]{
	\includegraphics[width=0.3\textwidth]{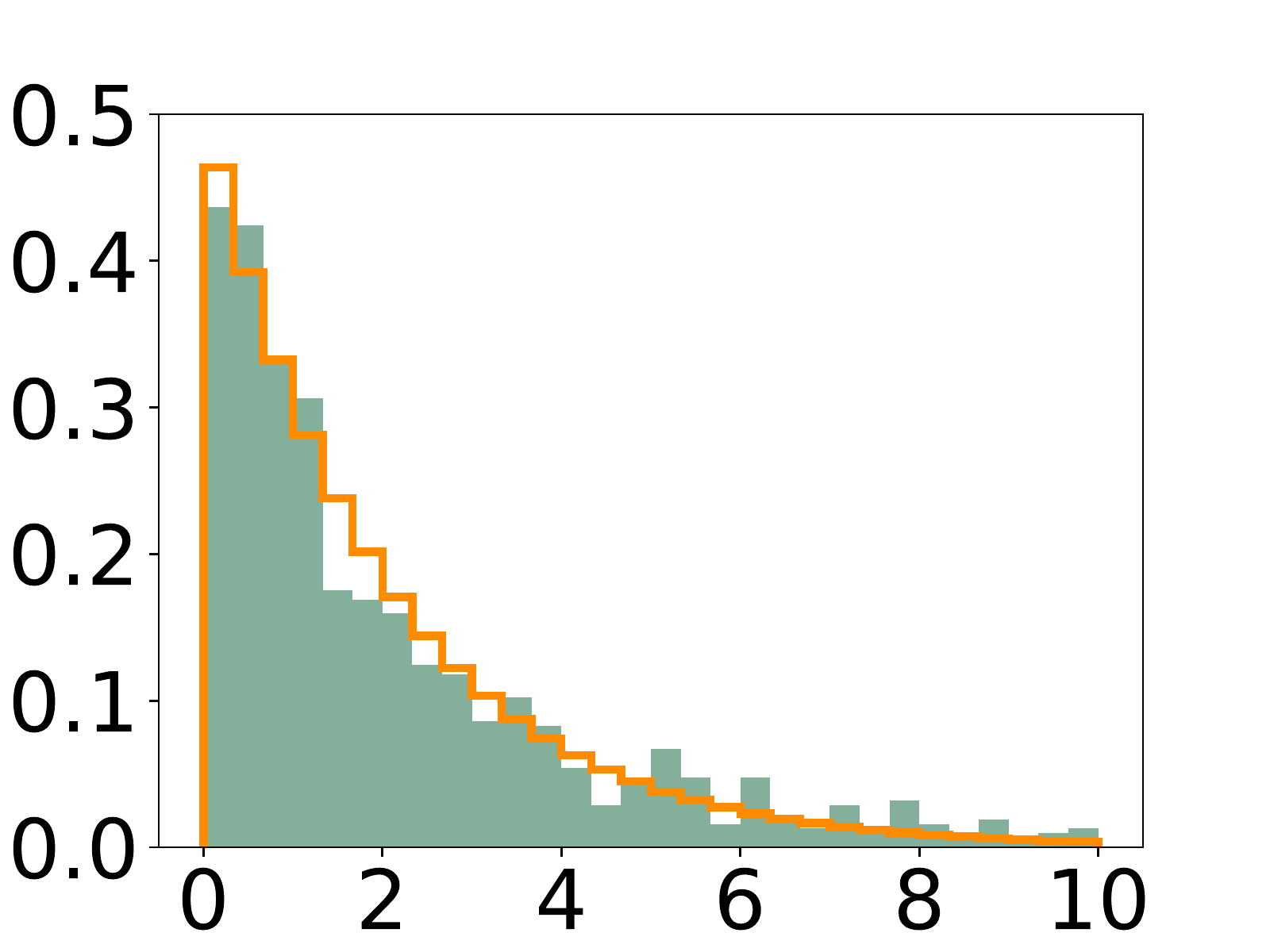}
}
\vspace{-5pt}
    \caption{Test for parameters under non-linear scenarios. The orange line is the theoretical true distribution. The baseline (linear regression) is more fat-tailed compared with Ours(NN) and Ours(L).}
    
\vspace{-20pt}
    \label{Fig: test_x_2} 
\end{figure} 

The phenomenon shown in Figure~\ref{Fig: test_x_2} leads to a fake significance test results! 
This fake fat-tailed distribution will make more variables determined to be significant incorrectly. 
For instance, when we set significance level as $\alpha = 0.05$ (which means that the parameter can be determined incorrectly with probability around 0.05), $w_1$ is determined incorrectly in linear regression (LSE) with probability $0.1263 (\pm 0.0137)$, while Ours(L) with probability $0.0715 (\pm 0.0024)$. We repeat the experiments six times, and the brackets show a $95\%$ confidence interval. 
The results of Ours(L) make the significance test more accurate.

We further show its $95\%$ confidence interval for correctness in Table~\ref{Tab: NonLinearCorrectness} quantitatively, where our newly proposed estimators work better. 
More details about Table~\ref{Tab: NonLinearCorrectness} are shown in Appendix~\ref{Appendix: supplement}.

\begin{table}[htbp]
\vspace{-15pt}
	\centering 
	\caption{Correctness Comparison (Non-Linear)}
	\label{Tab: NonLinearCorrectness} 
	\begin{tabular}{c|c|c|c}  
		\hline 
		\ & $\chi_2$ & normal ($w_0$)& normal ($w_1$) \\
		\hline
		Linear Reg & 0.1150 ($\pm 0.0101$)& 0.0635 ($\pm 0.0062$)& 0.0715 ($\pm 0.0070$)\\
		Ours(NN) & {\color{red} 0.0679 ($\pm 0.0054$)} & {\color{red}0.0326 ($\pm 0.0058$)} & 0.0650 ($\pm 0.0054$)\\		
		Ours(L) & 0.0810 ($\pm 0.0067$) & 0.0489 ($\pm 0.0077$)& {\color{red} 0.0276 ($\pm 0.0046$)} \\
		\hline
    \end{tabular}
\vspace{-5pt}
\end{table}

We compare the efficiency of Ours(NN) and Ours(L) in Table~\ref{Tab: NonLinearEfficiency} with its $95\%$ confidence interval. 
Note that since linear regression returns a wrong asymptotic efficiency, it is not listed here.


\begin{table}[htbp]
\vspace{-15pt}
	\centering 
	\caption{Efficiency Comparison (Non-Linear)}
	\label{Tab: NonLinearEfficiency} 
	\begin{tabular}{c|c|c}  
		\hline 
		\ & $w_0$ &  $w_1$  \\
		\hline
		Ours(NN) & 0.0214($\pm 0.0002$)  & 0.0593($\pm 0.0026$) \\		
		Ours(L) & 0.0328($\pm 0.0002$) & 0.0604 ($\pm 0.0003$) \\		
		\hline
    \end{tabular}
\vspace{-10pt}
\end{table}

\section{Conclusion}
In this paper, we propose a new estimator for the linear coefficient that works well in both linear and non-linear cases.
Unlike traditional statistical inference methods, machine learning models are introduced to the significance test process.
For future work, our new framework may be extended to the more general statistical inference scenarios (e.g. high-dimensional settings), and it will be interesting to show how machine learning models affect the efficiency of our estimator.

\section*{Broader Impact}
Compared with the traditional significance test, our new methods output a more precise significance level (or p-value) when linear assumptions do not hold. Moreover, with small efficiency loss, one can better extract the relationship between explanatory variables and response variables. Therefore, our estimator might be a better tool for the significance test. 



\begin{ack}
We are grateful to Yang Bai, Chenwei Wu for helpful comments on an early draft of this paper.
This work has been partially supported by Shanghai Qi Zhi Institute, Zhongguancun Haihua Institute for Frontier Information Technology, the Institute for Guo Qiang Tsinghua University (2019GQG1002), and Beijing Academy of Artificial Intelligence. 
\end{ack}

\bibliographystyle{named}
\bibliography{references} 
\newpage
\appendix

\section{Proofs}
\label{appendix: proof}

\subsection{The Proof of Lemma~\ref{Lem: BoundForLoss}}
This proof is mainly based on Bayesian Estimation, where we have the prior information that $R \sim U(0,1)$.

For a given $t \geq UB^u_{S_2}-LB^u_{S_2}$, we have

\begin{equation*}
    \begin{split}
        & \mathbb{P}(R\leq t | u_1, u_2, \dots, u_{N_2}) \\
        =& \frac{\int_{[0,t]} \mathbb{P}(u_1, u_2, \dots, u_{N_2} | R=s) \rm ds}{\int_{[0,1]} \mathbb{P}(u_1, u_2, \dots, u_{N_2} | R=s) \rm ds}
        \\
        =& \frac{\int_{UB^u_{S_2}-LB^u_{S_2}}^{t}1/s^{N_2} \rm ds }{\int_{UB^u_{S_2}-LB^u_{S_2}}^{1}1/s^{N_2} \rm ds} \\
        =& \frac{\left(UB_{S_2}^u-LB^u_{S_2}\right)^{-N_2+1} - t^{-N_2+1}}{\left(UB_{S_2}^u-LB^u_{S_2}\right)^{-N_2+1} - 1} 
        \end{split}
\end{equation*}

The first equation is due to its definition. The second equation is because the probability is zero when $s \leq UB_{S_2}^u-LB^u_{S_2}$.
Denote $c = UB_{S_2}^u-LB^u_{S_2}$. 
By setting $t = \frac{c}{[1-(1-\delta_0)(1-c^{N_2-1})]^{1/N_2-1}}$, we have
$$
\mathbb{P}(R\leq \frac{c}{[1-(1-\delta_0)(1-c^{N_2-1})]^{1/N_2-1}} | u_1, u_2, \dots, u_{N_2}) \geq 1-\delta_0
$$

We slightly enlarger $t$ with $\frac{c}{[1-(1-\delta_0)(1-c^{N_2-1})]^{\frac{1}{N_2-1}}} \leq c \delta_0^{-\frac{1}{N_2-1}}$, which leads to the results that

$$
\mathbb{P}(R\leq (UB_{S_2}^u-LB^u_{S_2})\delta_0^{-\frac{1}{N_2-1}} | u_1, u_2, \dots, u_{N_2}) = 1-\delta_0
$$

The proof is done.

\subsection{The Proof of Theorem~\ref{Thm: MSE}}
In this section, we will prove Theorem~\ref{Thm: MSE}. Before the proof, we propose Lemma~\ref{Lem: Independent} first, which focuses on why we need to split the datasets into the training set and the validation set.

\begin{lemma}[Independence Lemma]
\label{Lem: Independent}
If $x^{(i)}$ and $x^{(j)}$ are independent random variables, and $f$ is a fixed function which is independent of $x^{(i)}$ and $x^{(j)}$, then $f\left( x^{(i)} \right)$  is independent of $f\left( x^{(j)} \right)$ . 
\end{lemma}

\begin{proof}[Proof of Lemma~\ref{Lem: Independent}]
The proof directly follows the definition of independence of random variables. 
\begin{equation}\begin{aligned}
& \mathbb{P}\left(f\left(x^{(i)}\right) \leq u_{i}, f\left(x^{(j)}\right) \leq u_{j}\right)\\
=& \mathbb{P}\left(x^{(i)} \in C_{i}, x^{(j)} \in C_{j}\right)  \\
=& \mathbb{P}\left(x^{(i)} \in C_{i}\right) P\left(x^{(j)} \in C_{j}\right)\\
=& \mathbb{P}\left(f\left(x^{(i)}\right) \leq u_{i}\right) \mathbb{P}\left(f\left(x^{(j)}\right) \leq u_{j}\right)
\end{aligned}\end{equation}
where $C_i$, $C_j$ is the corresponding measurable sets which is decided by $f$ and $u_i$, $u_j$. The second equality follows the independence of $x^{(i)}$ and $x^{(j)}$. By definition, $f\left( x^{(i)} \right)$  is independent of $f\left( x^{(j)} \right)$ . The proof is done.
\end{proof}

The next corollary~\ref{Cor: Independence} is a direct application of Lemma~\ref{Lem: Independent}.
\begin{corollary}[Random Split of Datasets]
\label{Cor: Independence}
Given i.i.d. data $\left\{\left(x^{(i)}, y^{(i)}\right)\right\}_{i=1}^{n}$ which is randomly split into training data $S_1$ and test data $S_2$. If we use $S_1$ to train a machine learning model $f^{1}$, then for two independent samples $\left(x^{(i)}, y^{(i)}\right)$, $\left(x^{(j)}, y^{(j)}\right)$, and loss of $i_{th}$ sample $l^{(i)}=\left(f^{1}\left(x^{(i)}\right)-y^{(i)}\right)^{2}$, $l^{(i)}$ is independent of $l^{(j)}$.
\end{corollary}

It should be noted that $f^{1}$ is trained by $S_1$, thus is independent of samples $\left(x^{(i)}, y^{(i)}\right)$ in the validation set. 
That is why the dataset needs to be randomly split. Armed with Corollary~\ref{Cor: Independence}, we can go on to finish the proof.

We split the proof into two parts. In Lemma~\ref{Lem: Approximation}, we give an approximation measure for the machine learning model $f^1$. In Lemma~\ref{Lem: LinearApproach}, we will prove that the linear approximations of two close functions are also close. 
In this part, we use MSE to measure how the machine learning model approaches ground truth functions.
\begin{lemma}[Approximation measure for $f^1$]
\label{Lem: Approximation}
Given an i.i.d. dataset $S=\left\{\left(x_{i}, y_{i}\right)\right\}_{i=1}^{n}$ which is split into training set $S_1$ and validation set $S_2$, where $\left|S_{1}\right|=N_{1},\left|S_{2}\right|=N_{2}$. Suppose we use MSE ($l(x,y)=(f^1(x)-y)^2$) as our loss, and the sample loss is denoted as $l^{(i)}$, the population loss is denoted as $L$. 
Then for a given $\delta \in (0,1)$, we have
$$\mathbb{P}\left(L \geq \frac{1}{N_{2}} \sum_{i=1}^{N_{2}} l^{(i)}+(U B^{l}-L B^{l}) \sqrt{\frac{\log \left(\frac{1}{\delta}\right)}{2 N_{2}}}\right) \leq \delta$$
Namely, $w.p.>1-\delta$, $\left\|f^{1}-f^{*}\right\|_{2}^{2}= L \leq \frac{1}{N_{2}} \sum_{i=1}^{N_{2}} l^{(i)}+(U B^{l}-L B^{l}) \sqrt{\frac{\log \left(\frac{1}{\delta}\right)}{2 N_{2}}}$.
\end{lemma}

Note that Lemma~\ref{Lem: Approximation} gives an measure for $\left\|f^{1}-f^{*}\right\|_{2}$, which is a probabilistic upper bound. And the results also show a trade-off between training set and validation set. If more data are split into training set, $\frac{1}{N_{2}} \sum_{i=1}^{N_{2}} l^{(i)}$ decreases theoretically. If more data are split into validation set, $\sqrt{\frac{\log \left(\frac{1}{\delta}\right)}{2 N_{2}}}$ decreases theoretically. 

\begin{proof}[Proof of Lemma~\ref{Lem: Approximation}]
The key of the proof is Hoeffding's inequality, which states that given a series of bounded random variables $X_i, i=1,2,\dots,N$, and $X_i\in[t_1, t_2]$, if $\mathbb{E} X_i = \mathbb{E}X$, then for any $t>0$, we have

$$\mathbb{P}\left(\sum_{i=1}^{N}\left(X_{i}-\mathbb{E} X\right) \geq t\right) \leq \exp \left(-\frac{2 t^{2}}{N(t_2-t_1)^{2}}\right)$$

Plug the bound of loss $l^{(i)}$ into this inequality. By setting $t = (UB^{l}-LB^{l}) \sqrt{\frac{N_{2} \log \frac{1}{\delta}}{2}}$, it holds that
 $$\mathbb{P}\left(l \geq \frac{1}{N_{2}} \sum_{i=1}^{N_{2}} l^{(i)}+(U B^{l}-L B^{l}) \sqrt{\frac{\log \left(\frac{1}{\delta}\right)}{2 N_{2}}}\right) \leq \delta$$
 
Finally, notice that 
\begin{equation*}
    \Vert f^1-f^* \Vert_2^2 = \int (f^1(x)-y)^2 \rm d\mathbb{P} = L
\end{equation*}
The proof is done.
\end{proof}

The next Lemma~\ref{Lem: LinearApproach} shows that when the distance of two functions is bounded, the distance of their linear approximation is also bounded. 

\begin{lemma}
\label{Lem: LinearApproach}
Given oracle model $f^*$ and machine learning model $f^1$, where their linear approximation is $g^{*}= w^*x = \operatorname{argmin}_{g \in \mathcal{G}}\left\|g-f^{*}\right\|_{2}$, $g^{1}=w^1x = \operatorname{argmin}_{g \in \mathcal{G}}\left\|g-f^{1}\right\|_{2}$, respectively, where $\mathcal{G}=\left\{w^{T} x, w \in R^{d}\right\}$ is a linear function family. Given $m I_{d} \leqslant [\mathbb{E}\left(x x^{T}\right)]^{-1} \leqslant M I_{d}$, if $\left\|f^{*}-f^{1}\right\| \leq \epsilon$, then 

\begin{equation*}
\begin{split}
    \left\|g^{*}-g^{1}\right\| &\leq  \sqrt{d} \frac{M}{m} \epsilon \\
    \left\|w^{*}-w^{1}\right\| &\leq  \sqrt{d} \frac{M}{\sqrt{m}} \epsilon
\end{split}    
\end{equation*}

where d is the dimension of $x$.
\end{lemma}

\begin{proof}[Proof of Lemma~\ref{Lem: LinearApproach}]
First, we would like to represent $g$ in a linear form. For simplification, we omit the superscript for a while.
$$g=\operatorname{argmin}_{g \in \mathcal{G}}\|g-f\|_{2}=\operatorname{argmin}_{g \in \mathcal{G}}\|g-f\|_{2}^{2}=w^{T} x $$

Then we can derive an explicit representation of $w$
$$\begin{aligned}
w &=\operatorname{argmin}_{w} \mathbb{E}_{x}\left\|w^{T} x-f\right\|_{2}^{2}  \\
&=\operatorname{argmin}_{w} w^{T} \mathbb{E}\left(x x^{T}\right) w-2 w^{T} \mathbb{E}(x f(x))+f^{T} f  \\
&=\left[\mathbb{E}\left(x x^{T}\right)\right]^{-1} \mathbb{E}[x f(x)]
\end{aligned}$$

Therefore, by adding the superscript, we have
\begin{equation}
\label{Eqn: differOfw}
   w^{*}-w^{1}=\left[\mathbb{E}\left(x x^{T}\right)\right]^{-1} \mathbb{E}\left[x\left(f^{*}(x)-f^{1}(x)\right)\right] 
\end{equation}

It can be further calculated that 
$$\begin{aligned}
\left\|w^{*}-w^{1}\right\|_{2}^{2} & \leq\left\|\left[\mathbb{E}\left(x x^{T}\right)^{-1}\right]\right\|_{2}^{2}\left\|\mathbb{E}\left[x\left(f^{*}(x)-f^{1}(x)\right)\right]\right\|_{2}^{2} \\
&\left.=\left\|\left[\mathbb{E}\left(x x^{T}\right)^{-1}\right]\right\|_{2}^{2}\left[\sum_{i} \mathbb{E}\left(x_{i}\left(f^{*}(x)-f^{1}(x)\right)\right)^{2}\right]\right\} \\
& \leq\left\|\left[\mathbb{E}\left(x x^{T}\right)^{-1}\right]\right\|_{2}^{2}\left(\sum_{i} \int x_{i}^{2} d \mathbb{P} \int\left(f^{*}(x)-f^{1}(x)\right)^{2} d \mathbb{P}\right)^{2} \\
&=\left\|\left[\mathbb{E}\left(x x^{T}\right)^{-1}\right]\right\|_{2}^{2}\left(\sum_{i} \int x_{i}^{2} d \mathbb{P}\left\|f^{*}(x)-f^{1}(x)\right\|_{2}^{2}\right)\\
& \leq \epsilon^{2}\left\|\left[\mathbb{E}\left(x x^{T}\right)^{-1}\right]\right\|_{2}^{2}\left(\sum_{i} \mathbb{E}\left(x_{i}^{2}\right)\right)
\end{aligned}$$
where the first inequality comes from the definition of matrix norm, the third inequality is due to Cauchy-Schwartz inequality, and the final equality is by the bound of functions $f^*$ and $f^1$.

If eigenvalues of $A=[\mathbb{E}\left(x x^{T}\right)]^{-1}$ is denoted by $m = \lambda_d \leq \lambda_{d-1} \leq \dots \leq \lambda_1 =M$, then 
$$\begin{aligned}
&\left\|\mathbb{E}\left(x x^{T}\right)^{-1}\right\|_{2}=\lambda_1\\
&\sum_{i} E\left(x_{i}^{2}\right)=\operatorname{tr}\left[\mathbb{E}\left(x x^{T}\right)\right]=\sum_i \frac{1}{\lambda_{i}}
\end{aligned}$$

That is to say,
\begin{equation}
\label{Eqn: bound for w}
\left\|w^{*}-w^{1}\right\|_{2} \leq \epsilon \lambda_1 \sqrt{\sum_{i} \frac{1}{\lambda_{i}}}
\end{equation}

Plug Equation~\ref{Eqn: bound for w} into $g(x)=w^T x$,
\begin{equation*}\begin{aligned}
\left\|g^{*}-g^{1}\right\|_{2}^{2} &=\left\|\int\left(w^{* T} x-w^{1 T} x\right)^{2} d \mathbb{P}\right\|_{2} \\
&=\left\|\left(w^{*}-w^{1}\right)^{T} \int x x^{T} d \mathbb{P}\left(w^{*}-w^{1}\right)\right\|_{2} \\
& \leq\left\|w^{*}-w^{1}\right\|_{2}^{2}\left\|\mathbb{E}\left(x x^{T}\right)\right\|_{2}
\end{aligned}\end{equation*}

Therefore, we have
\begin{equation}
\label{Eqn: bound for g}
\left\|g^{*}-g^{1}\right\|_{2} \leq\left\|w^{*}-w^{1}\right\|_{2} \sqrt{\frac{1}{\lambda_d}} \leq \epsilon \lambda_1 \sqrt{\sum_{i} \frac{1}{\lambda_{i}}} \sqrt{\frac{1}{\lambda_d}} 
\end{equation}

Plug $\sum_i \frac{1}{\lambda_i} \leq d \frac{1}{\lambda_d}$ into Equation~\ref{Eqn: bound for w} and Equation~\ref{Eqn: bound for g}, we have
\begin{equation}\begin{array}{c}
\left\|w^{*}-w^{1}\right\|_{2} \leq \epsilon \sqrt{d} \frac{M}{\sqrt{m}} \\
\left\|g^{*}-g^{1}\right\|_{2} \leq \epsilon \sqrt{d} \frac{M}{m}
\end{array}\end{equation}

The proof is done.
\end{proof}

Theorem~\ref{Thm: MSE} is a direct combination of Lemma~\ref{Lem: Approximation} and Lemma~\ref{Lem: LinearApproach}. 

\subsection{The Proof of Corollary~\ref{Cor: MSE}}
Corollary~\ref{Cor: MSE} directly follows Theorem~\ref{Thm: MSE}.
From Theorem~\ref{Thm: MSE} we know that 
\begin{equation*}
\begin{split}
\left\|g^{*}-g^{1}\right\|_{2} &\leq \sqrt{d} \frac{M}{m} \epsilon\\
\left\|w^{*}-w^{1}\right\|_{2} &\leq \sqrt{d} \frac{M}{\sqrt{m}} \epsilon 
\end{split}
\end{equation*}
where $\epsilon = \sqrt{\frac{1}{N_{2}} \sum_{i=1}^{N_{2}} l^{(i)}+(U B^{l}-L B^{l}) \sqrt{\frac{\log \left(\frac{1}{\delta}\right)}{2 N_{2}}}}$.
By Lemma~\ref{Lem: BoundForLoss}, we use empirical bounds to replace the real bounds. 
That is to say, we use $(U B_{S_2}^{l}-L B_{S_2}^{l})\delta_0^{-\frac{1}{N_2-1}}$ to replace $U B^{l}-L B^{l}$.
This operation brings loss, the probability becomes $(1-\delta)(1-\delta_0)$.
By the following inequality, we slightly reduce the probability.
$$(1-\delta)(1-\delta_0) \geq 1-\delta-\delta_0$$
So we have, $w.p.\geq 1-\delta-\delta_0$,
\begin{equation*}
\begin{split}
\left\|g^{*}-g^{1}\right\|_{2} &\leq \sqrt{d} \frac{M}{m} \epsilon_{S_2}\\
\left\|w^{*}-w^{1}\right\|_{2} &\leq \sqrt{d} \frac{M}{\sqrt{m}} \epsilon_{S_2} 
\end{split}
\end{equation*}
The proof is done.

\subsection{The Proof of Lemma~\ref{Lem: covariance}}
The key point of the proof comes from \citet{vershynin_2018} (Page 130). 
They states that under the assumptions of Lemma~\ref{Lem: covariance}, the following inequality holds $w.p. \geq 1-\delta/2$
\begin{equation}
\label{Eqn: difference of A}
\Vert \hat{A}_n^{-1} - A^{-1} \Vert \leq C \Vert A^{-1} \Vert \left( \sqrt{\frac{K^2 d \log{(4d/\delta)}}{n}} +\frac{K^2 d \log{(4d/\delta)}}{n} \right) 
\end{equation}

Therefore, under the assumption that $m I_{d} \preccurlyeq A, \hat{A}_{N_t} \preccurlyeq M I_{d}$, we have
\begin{equation}
\label{Eqn: bound for A}
\Vert{A}^{-1} \Vert_2 \leq \frac{1}{m}
\end{equation}

Plug Equation~\ref{Eqn: bound for A} into the Equation~\ref{Eqn: difference of A}, $w.p. \geq 1-\delta/2$:
\begin{equation*}
\begin{split}
\Vert \hat{A}_n - A \Vert_2 &= 
\Vert(\hat{A}_n)(\hat{A}_n^{-1} - A^{-1}) (A) \Vert_2  \\
&\leq \Vert\hat{A}_n \Vert_2  \Vert A \Vert_2  \Vert\hat{A}_n^{-1} - A^{-1} \Vert_2 \\
&\leq C \frac{M^2}{m} \left( \sqrt{\frac{K^2 d \log{(4d/\delta)}}{n}} +\frac{K^2 d \log{(4d/\delta)}}{n} \right)
\end{split}
\end{equation*}

The proof of Lemma~\ref{Lem: covariance} is done.

\subsection{The Proof of Theorem~\ref{Thm: fillingTheBias}}

By denoting $A \triangleq [\mathbb{E}(xx^T)]^{-1}$ and $z \triangleq x e = x\left(f^*\left(x\right) - f^1\left(x\right)\right)$, where $e$ is the residual, the result of Equation~\ref{Eqn: differOfw} can be rewritten as
$$
w^{*} - w^{1} = A \mathbb{E}(z)
$$

Then the $j_{t h}$ coefficient $w_j$ can be represented as
$$
\left( w^*-w^1 \right)_{j} = \mathbb{E}A^j z
$$
where $A^j$ is the $j_{t h}$ row of $A$. Denote the realization of $A^j z$ on $i_{t h}$ sample as $z^{(i)} \triangleq x^{(i)}e^{(i)}$. 
Since $A^j z$ is bounded by $LB^{A^j z} \leq A^j z \leq UB^{A^j z}$, by using two-side Hoeffding's inequality, we have
\begin{equation}
    \label{Eqn: hoff for Ajz}
\mathbb{P}\left(\left|\mathbb{E} A^{j} z-\frac{1}{N_{2}} \sum_{i} A^{j} z^{(i)}\right| \geq(U B^{A^j z}-L B^{A^j z}) \sqrt{\frac{\log \frac{4}{\delta}}{2 N_{2}}}\right) \leq \delta/2    
\end{equation}

We need to further replace $A^j$ with $\hat{A}^j_{N_t}$, which holds that $w.p. \geq 1-\delta/2$
\begin{equation}
\label{Eqn: differ of Az}
    \begin{split}
\left\| \frac{1}{N_{2}} \sum_{i=1}^{N_2} (\hat{A}_{N_t} - {A}) z^{(i)} \right\|_2 &=\left\| \hat{A}_{N_t} - {A} \right\|_2 \left\| \frac{1}{N_{2}} \sum_{i=1}^{N_2} z^{(i)} \right\|_2 \\
&\leq C \frac{M^2}{m} \Vert \frac{1}{N_{2}} \sum_{i=1}^{N_2} z^{(i)} \Vert_2 \left( \sqrt{\frac{K^2 d \log{(4d/\delta)}}{N_t}} +\frac{K^2 d \log{(4d/\delta)}}{N_t} \right)
    \end{split}
\end{equation}

Combining Equation~\ref{Eqn: hoff for Ajz} and Equation~\ref{Eqn: differ of Az}, the following inequality holds $w.p. \geq 1-\delta$

\begin{equation*}
\begin{split}
\left|w^{*}_{j}-w^{1}_{j}-\frac{1}{N_{2}} \sum_{i=1}^{N_2} \hat{A}^{j}_{N_t} z^{(i)}\right| \leq & \left|w^{*}_{j}-w^{1}_{j}-\frac{1}{N_{2}} \sum_{i=1}^{N_2} A^j z^{(i)}\right| + \left|\frac{1}{N_{2}} \sum_{i=1}^{N_2} (\hat{A}^{j}_{N_t} - {A}^j) z^{(i)} \right| \\
\leq& \left| \mathbb{E}A^j z -\frac{1}{N_{2}} \sum_{i=1}^{N_2} A^j z^{(i)}\right| + \left \| \frac{1}{N_{2}} \sum_{i=1}^{N_2} (\hat{A}_{N_t} - {A}) z^{(i)} \right\|_2 \\
\leq & \left( U B_{S_2}^{A^j z}-L B_{S_2}^{A^j z} \right) \sqrt{\frac{\log \frac{4}{\delta}}{2 N_{2}}} \\
& + C \frac{M^2}{m} \Vert \frac{1}{N_2} \sum_{i=1}^{N_2} z^{(i)}\Vert  \left( \sqrt{\frac{K^2 d \log{(4d/\delta)}}{N_t}} +\frac{K^2 d \log{(4d/\delta)}}{N_t} \right)
\end{split}    
\end{equation*}
The first inequality is the triangle inequality, and the second inequality is from the definition of vector norm. 

The proof is done.

\subsection{The Proof of Corollary~\ref{Cor: fillingTheBias}}
We mainly use $UB_{S_2}^{\hat{A}^j_{N_t}z} -LB_{S_2}^{\hat{A}^j_{N_t}z}$ to replace $UB^{A^jz} -LB^{{A}^jz}$.

Firstly, notice that $w.p. \geq 1-\delta_0$,
$$ 
UB^{A^jz}- LB^{A^jz} \leq (UB^{A^jz}_{S_2}- LB_{S_2}^{A^jz})\delta_0^{-\frac{1}{N_2-1}}
$$

We will bound $UB^{A^jz}_{S_2}- LB_{S_2}^{A^jz}$ by $UB^{\hat{A}^j_{N_t}z} -LB^{\hat{A}^j_{N_t}z}$ in the following.
First, we know from  Lemma~\ref{Lem: covariance} that $\hat{A}_{N_t}$ is close to ${A}$. So we have for every $z \in S$ $w.p.\geq 1-\delta$
\begin{equation}
\begin{split}
    \vert \hat{A}_{N_t}^j z - A^jz  \vert &\leq \Vert \hat{A}_{N_t} z - Az \Vert \\
    &\leq \Vert \hat{A}_{N_t} - A \Vert \Vert z\Vert\\
    &\leq b_0 \Vert z \Vert\\
    &\leq b
\end{split}
\end{equation}
where $b_0=C \frac{M^2}{m} \left( \sqrt{\frac{K^2 d \log{(4d/\delta)}}{n}} +\frac{K^2 d \log{(4d/\delta)}}{n} \right)$,
$b = b_0 \max_{S_2} \Vert z \Vert$.

Denote $UB_{S_2}^{A^jz} = A^jz_{(u)}, LB_{S_2}^{A^jz} = A^jz_{(l)}$, then $w.p.\geq 1-2\delta$
\begin{equation*}
\begin{split}
    UB_{S_2}^{A^jz} -LB_{S_2}^{A^jz} &= A^jz_{(u)} - A^jz_{(l)} \\
    &\leq (\hat{A}_{N_t}^jz_{(u)}+b) - (A^j_{N_t}z_{(l)}-b) \\
    &= (\hat{A}_{N_t}^jz_{(u)} - A^j_{N_t}z_{(l)})+2b \\
    &\leq UB_{S_2}^{\hat{A}^j_{N_t}z} -LB_{S_2}^{\hat{A}^j_{N_t}z} + 2b
\end{split}
\end{equation*}
The final inequality is from its definition.

Combining them together, $w.p.\geq 1-\delta_0-2\delta$
\begin{equation*}
    UB^{A^jz}- LB^{A^jz} \leq (UB^{\hat{A}_{N_t}^jz}_{S_2}- LB_{S_2}^{\hat{A}_{N_t}^jz} + 2b)\delta_0^{-\frac{1}{N_2-1}}
\end{equation*}

Moreover, we know that 
\begin{equation*}
    \begin{split}
\Vert \frac{1}{N_2} \sum z^{(i)} \Vert &\leq  \frac{1}{N_2} \sum \Vert z^{(i)} \Vert \\
&\leq max_{S_2}\Vert z \Vert
    \end{split}
\end{equation*}
where $z \in S_2$.

Therefore, 
\begin{equation*}
     b_0 \Vert \frac{1}{N_2} \sum z^{(i)} \Vert \leq b_0 \max_{S_2} \Vert z \Vert = b
\end{equation*}
Plug them into Theorem~\ref{Thm: fillingTheBias}, the proof is done.

\subsection{The Proof of Theorem~\ref{Thm: Asymptotic}}
The key point of this proof is Slutsky's theorem, which is stated as follows

\begin{lemma}[Slutsky's theorem]
Let $\{X_n\}$ and $\{Y_n\}$ be sequence of scalar / vector / matrix random elements. If $X_n$ converges in distribution to a random element $X$, and $Y_n$ converges to a constant $c$, then
\begin{itemize}
    \item $X_n + Y_n \stackrel{d}\to X+c$
    \item $X_n Y_n \stackrel{d} \to c X$
    \item $X_n / Y_n \stackrel{d} \to X/c$ provided that $c$ is invertible
\end{itemize}
\end{lemma}

Furthermore, consider two estimators $w^a = w^1 + \frac{1}{N_2} \sum_{i=1}^{N_2} A z^{(i)}$ and $w^e = w^1 + \frac{1}{N_2} \sum_{i=1}^{N_2} \hat{A}_{N_t} z^{(i)}$. The aim of Theorem~\ref{Thm: Asymptotic} is to derive the asymptotic distribution of $w^e$. 

The main intuition is: $w^a$ can reach asymptotic normal distribution due to Central Limit Theorem. 
So we want to prove the difference of $w^a$ and $w^e$ is sufficiently small, and converges to zero in probability.
By Slutsky's Theorem, the distribution of $w^e$ is then asymptotic normal distribution. 
But notice that $w^a$ and $w^e$ share different variance (although they converge to the same constant in probability), which will be discussed later. 


We first derive the asymptotic distribution of $w^a$ based on Central Limit Theorem. But before that, we now focus on the distribution of $\sum_i z^{(i)}$. 
By Multidimensional Central Limit Theorem, consider a series of {i.i.d.} random vectors $z^{(i)}$, their mean converges to the normal distribution
$$
\frac{1}{\sqrt{N_2}} \sum_{i=1}^{N_2} \left(z^{(i)}-\mathbb{E}z\right) \stackrel{d} \sim \mathcal{N}(0, \mathbb{D}z)
$$

Furthermore, $w^a$ can be regarded as a transformation of $\sum_i{z^{(i)}}$, with mean
$\mathbb{E}w^a = w^1 + \mathbb{E}A z = w^*$
and variance 
$\mathbb{D}w^a = \frac{1}{N_2}  A \left(\mathbb{D} z \right) A \triangleq \Sigma^2_e $.
We also denote $D \triangleq \mathbb{D} z $ for simplification.

Therefore the following asymptotic property holds
\begin{equation}
    \label{Eqn: asymptotic_wa}
\Sigma^{-1}_e (w^{a} - w^{*}) \stackrel{d} \sim \mathcal{N}\left(0, I_d\right)
\end{equation}

Next we prove the asymptotic property of $w^e$. 
Notice that the relationship between $\hat{\Sigma}^{-1}_e (w^{e} - w^{*})$ and $ \Sigma^{-1}_e (w^{a} - w^{*})$ can be stated as follows, where $\hat{\Sigma}_e^2 = \frac{1}{N_2} \hat{A}_{N_t} \hat{D} \hat{A}_{N_t} $

\begin{equation*}
    \begin{split}
        & \hat{\Sigma}^{-1}_e (w^{e} - w^{*}) - \Sigma^{-1}_e (w^{a} - w^{*}) \\
        =& \hat{\Sigma}_e^{-1}(w^e-w^a) + \left(\hat{\Sigma}_e^{-1} - \Sigma^{-1}_e\right)(w^a-w^*)
    \end{split}
\end{equation*}

First, we would like to show some bounded terms.
Since $x$ is bounded, and the weights in $f^1$ are all bounded, then $f^1(x)$ is bounded. And further we have $y$ is bounded, so $z$ is bounded.
Besides, since we assume $z$ is not always zero, $\mathbb{D}z$ and $\hat{\mathbb{D}}z$ are both two-sided bounded.
Since $\hat{A}_{N_t}$ and $A$ are both two-sided bounded, $\hat{\Sigma}_e^2$ and ${\Sigma}_e^2$ are two-sided bounded.

Besides, since $x$, $y$ are bounded, $w^*$ is bounded. And due to bounded $A, \hat{A}_{N_t}$, $w^a, w^e$ are also bounded.

We will prove this difference converges to zero in probability in the following. Consider the part $\hat{\Sigma}_e^{-1}(w^e-w^a)$ first. We know that 
$$
w^e - w^a = \left(\hat{A}_{N_t}-A \right) \left(\frac{1}{N_2} \sum_{i=1}^{N_2} z^{(i)}\right)
$$

By Lemma~\ref{Lem: covariance}, we know that $\lim_{N_t \to \infty} \mathbb{P} \left(\left\| \hat{A}_{N_t}-A \right\|\right)=0$ Therefore $\hat{A}_{N_t}$ converge to $A$ in probability.
$\hat{A}_{N_t} \stackrel{P}\to A$. Since $\left(\frac{1}{N_2} \sum_{i=1}^{N_2} z^{(i)}\right)$ is bounded, $w^e-w^a$ converges to zero in probability.
\begin{equation*}
    \hat w^e-w_a \stackrel{P}\to 0
\end{equation*}

Furthermore, since we assume $\hat{\Sigma}_e$ is double-side bounded in probability, then $\hat{\Sigma}_e^{-1}$ is also bounded in probability. So we have 
\begin{equation}
\label{Eqn: convergence of 1}
\hat{\Sigma}_e^{-1}(w^e-w^a) \stackrel{P} \to 0
\end{equation}

Next we consider $\left(\hat{\Sigma}_e^{-1} - \Sigma_e^{-1}\right)(w^a-w^*)$. Since the sample covariance converges to population covariance in probability according to Law of Large Numbers, we have
$$
\hat{D} \stackrel{P}\to D
$$

Combining the covergence of $\hat{A}_{N_t}$ and $\hat{D}$, by Slutsky's theorem,

$$\hat{A}_{N_t} \hat{D} \hat{A}_{N_t} \stackrel{P} \to A D A$$
which means that $\hat{\Sigma}^2_e \stackrel{P} \to \Sigma_e^2$. 
Since $\Sigma_e^2$, $\hat{\Sigma}_e^2$ are assumed to be two-side bounded, we can apply Continuous Mapping Theorem to get
$$
\hat{\Sigma}_e^{-1} \stackrel{P} \to \Sigma_e^{-1} 
$$

Since we further assume that $w^a$, $w^e$ are bounded in probability, we can say
\begin{equation}
    \label{Eqn: convergence of 2}
\left(\hat{\Sigma}_e^{-1} - \Sigma_e^{-1}\right) (w^a-w^*) \stackrel{P} \to 0
\end{equation}

Recall that 
$$
\hat{\Sigma}^{-1}_e (w^{e} - w^{*}) = \Sigma^{-1}_e (w^{a} - w^{*}) + \hat{\Sigma}_e^{-1}(w^e-w^a) + \left(\hat{\Sigma}_e^{-1} - \Sigma^{-1}_e\right)(w^a-w^*)
$$

Plug Equation~\ref{Eqn: convergence of 1} and Equation~\ref{Eqn: convergence of 2}, and apply Slutsky's theorem,
$$
\hat{\Sigma}_e^{-1} \left(w^e-w^*\right) \stackrel{d}\sim \mathcal{N}{(0, I_d)}
$$

The proof is done.

\section{Significance Test}
\label{appendix:test}

The significance test is a special type of hypothesis test in statistics, which focuses on whether a parameter is significantly different from zero.
In the hypothesis test process, we first assume $H_0$ holds, and then derive the corresponding distribution of certain statistics.
The asymptotic properties of an estimator can directly lead to the significance test.
For example, we can construct asymptotic normal or asymptotic $\chi^2$ distribution based on these asymptotic properties.

The significance test can be divided into \emph{model test} and \emph{coefficient test}.
The model test is used to determining whether a model is significantly different from a constant model, where a constant model means the model predicts each sample as a constant whatever its features are.
The model test can also be regarded as the test for a set of features.
The coefficient test is used to determining whether a feature can significantly affect its response variable.
Usually, we first do the model test. Once the model is significant, we then do the coefficient test.

We show the model test and coefficient test under significance level $\delta$ as follows, where $w^*$ represents the true parameter, $w^e$ represents the estimator of $w^*$, and $\hat{\Sigma}_e^2$ represents the estimator of the variance of $w^e$. 
Furthermore, $\chi^2_\delta$ is the $\delta$ quantile of $\chi^2$ distribution, and $u_{\delta/2}$ is the $\delta/2$ quantile of normal distribution.
As in the main text, we use $x_j$ to represent $j_{th}$ feature of $x$, and $\Sigma_{ij}$ to represent $i_{th}$ row and $j_{th}$ of a matrix $\Sigma$.

The following statements are all based on the asymptotic properties that
$$
\hat{\Sigma}_e^{-1} \left(w^e-w^*\right) \stackrel{d}\sim \mathcal{N}{(0, I_d)}
$$

\textbf{1. Model Test}
$$H_0: w^*=0\quad \text{v.s.} \quad H_1: w^* \not= 0$$
\begin{itemize}[leftmargin=0.3in]
    \item 
If $w^{e^{T}} (\hat{\Sigma}_{e}^2)^{-1} w^{e} \leq \chi_{\delta}^{2}$, we do not have enough confidence to reject the null hypothesis, which means the model is not significant.
\item 
If $w^{e^{T}} (\hat{\Sigma}_{e}^2)^{-1} w^{e} > \chi_{\delta}^{2}$, we have enough confidence to reject the null hypothesis, which means the model is significant.
\end{itemize}

\textbf{2. Coefficient Test}
$$H_0: w_j^*=0\quad \text{v.s.} \quad H_1: w_j^* \not= 0$$
\begin{itemize}[leftmargin=0.3in]
    \item 
If $\left|\frac{w_{j}^{e}}{\sqrt{(\hat{\Sigma}^2_{e})_{jj}}}\right| \leq u_{\frac{\delta}{2}}$, we do not have enough confidence to reject the null hypothesis, which means the coefficient is not significant.
\item 
If $\left|\frac{w_{j}^{e}}{\sqrt{(\hat{\Sigma}^2_{e})_{jj}}}\right| > u_{\frac{\delta}{2}}$, we have enough confidence to reject the null hypothesis, which means the coefficient is significant.
\end{itemize}

Notice that we use $\hat\Sigma^2$ instead of the true variance $\Sigma^2$, so the better choice of distribution in the coefficient test is \emph{t distribution}. But we still use the normal distribution for simplification, since when the number of samples is large enough, these two distributions are almost equivalent.

\section{Experiment Details}
\label{Sec: Experiment Details}

\subsection{Experiment Settings}
\label{Sec: Experiment Settings}
We will compare our proposed method with traditional Linear Regression. 
In linear regression, we focus on its LSE estimator, and use gradient descent to get $w^{LSE}$. 
In our new proposed method, we first split the datasets into training set (70\%) and validation set (30\%) and finally get $w^e$ based on the trained machine learning model (see Section~\ref{Sec: AsymptoticProperties}). 

In each case, we test whether the practical asymptotic distribution meets theory. We first propose the hypothesis test
$$H_0: w^*=w_t\quad \text{v.s.} \quad H_1: w^* \not= w_t$$
where $w^*$ is the true linear approximation parameter.

\textbf{Remark:} Here we use $w_t$ instead of $0$ (in Section~\ref{Sec: AsymptoticProperties}). 
That is because we aim to test whether $w^*=0$ in Section~\ref{Sec: AsymptoticProperties}, representing whether the parameter is significantly different from $0$.
And here we want to test whether the parameter is significantly different from the true parameter $w_t$.
They are fundamentally the same.

The asymptotic theory tells that $w^{LSE}$ and $w^e$ both reach some asymptotic distribution under $H_0$ (at least, we use this in practice).
Namely, as Section~\ref{Sec: AsymptoticProperties} shows, $\chi^2 = (w-w_t) \Sigma^{-1} (w-w_t) \stackrel{d}\sim \chi^2_{p}$ and $u_j = \frac{w_j-w_{tj}}{\sqrt{\Sigma_{jj}}} \stackrel{d}\sim \mathcal{N}(0,1)$, where $w$ is $w^{LSE}$ or $w^{e}$, and $\Sigma$ is the covariance of $w$.
We will test whether it is true in practice.

We conduct 1000 trials in each experiments.
Each time, we generate a new datasets $S$ with $|S| = 100$ and unlabelled data $S^u$ with $|S^u|=50000$. 
And then calculate $\hat{A} = [\hat{\mathbb{E}}(xx^T)]^{-1}$ with unlabelled data $S^u$, but in practice we can also use $S \cup S^u$ together.
We compute the correctness (KS statistics) and the efficiency (the average std) based on these 1000 trials.

Furthermore, we conduct each experiment repeatedly six times, and calculate the $95\%$ confidence interval of each statistics.
This will help reduce the randomness caused in the experiments. 

\subsection{Experimental Settings for Unbiasedness Test}
\label{appendix: unbiased}
In unbiasedness test (Section~\ref{Sec: Unbiaedness}), we would like to set $|S|=20$ with split rate $r = 0.5$, which means that we use 10 samples for training and 10 samples for validation.
We use a smaller dataset compared with Appendix~\ref{Sec: Experiment Settings} for better showing the bias.
The other hyperparameters will retain as in Appendix~\ref{Sec: Experiment Settings} (non-linear cases).

Each time, we calculate the mean of these 1000 trials.
After repeating this experiment six times, we calculate its $95\%$ confidence interval.
As shown in Section~\ref{Sec: Experiment}, LSE-based methods contain more biasness compared with our methods.

\section{Experiments Under Linear Settings}
\label{Sec: Exp_linear}
For completeness, we here show the comparison under linear scenarios.
We would show that our proposed method also works in a traditional linear scenario, which is 
\begin{equation*}
    f^*(x) = x + \epsilon, \epsilon\sim \mathcal{N}(0,\sigma^2), \epsilon \ \text{is independent of} \ x
\end{equation*}

Our aim is to estimate $w^* = (0,1)$, which can be written as the following hypothesis test
\begin{equation*}
    H_0: w^*=(0, 1)\quad \text{v.s.} \quad H_1: w^* \not= (0, 1)
\end{equation*}

Similar to Figure~\ref{Fig: test_x_2}, Figure~\ref{Fig: test_x} visualizes the simulation results and theoretical distribution here under linear case, and Table~\ref{Tab: LinearCorrectness} shows the same observations quantifiably. We see that three methods can all reach the true asymptotic distribution successfully. Also, the figure here is one group of the six.

\begin{figure}
    \centering
    \subfigure[Baseline: $w_0$]{
        \includegraphics[width=0.3\textwidth]{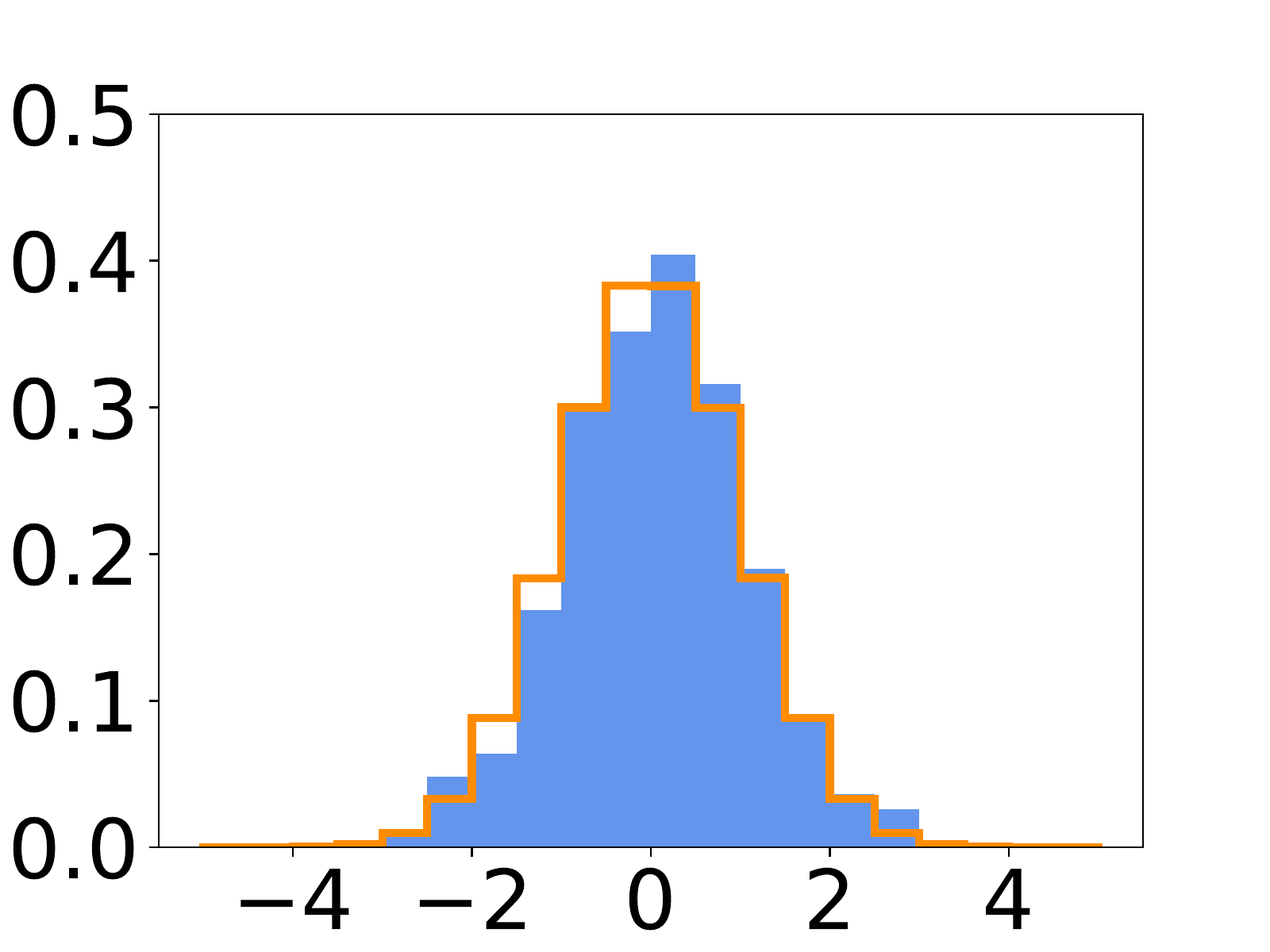}
    }
    \subfigure[Baseline: $w_1$]{
	\includegraphics[width=0.3\textwidth]{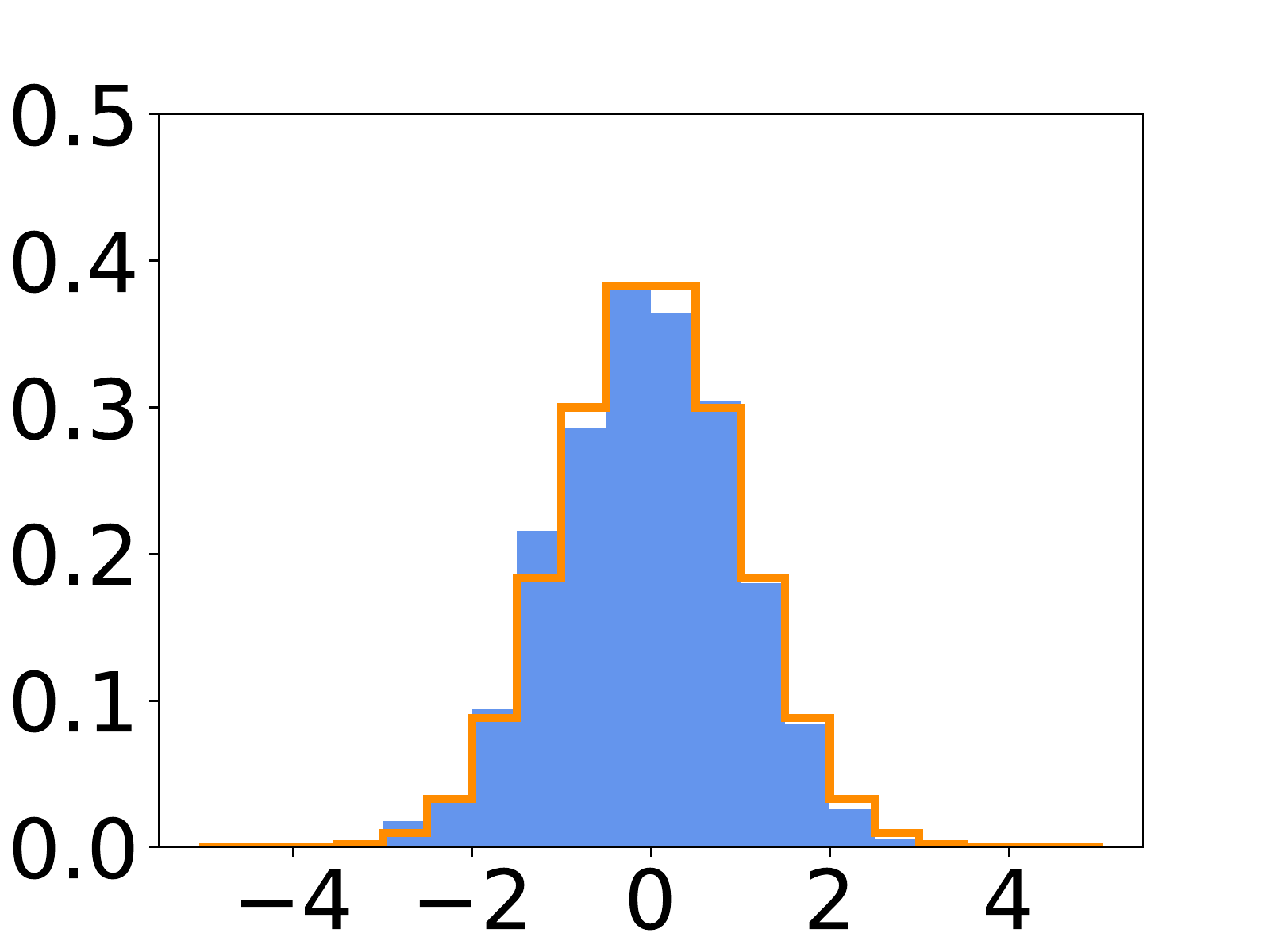}
}
    \subfigure[Baseline: $w$]{
	\includegraphics[width=0.3\textwidth]{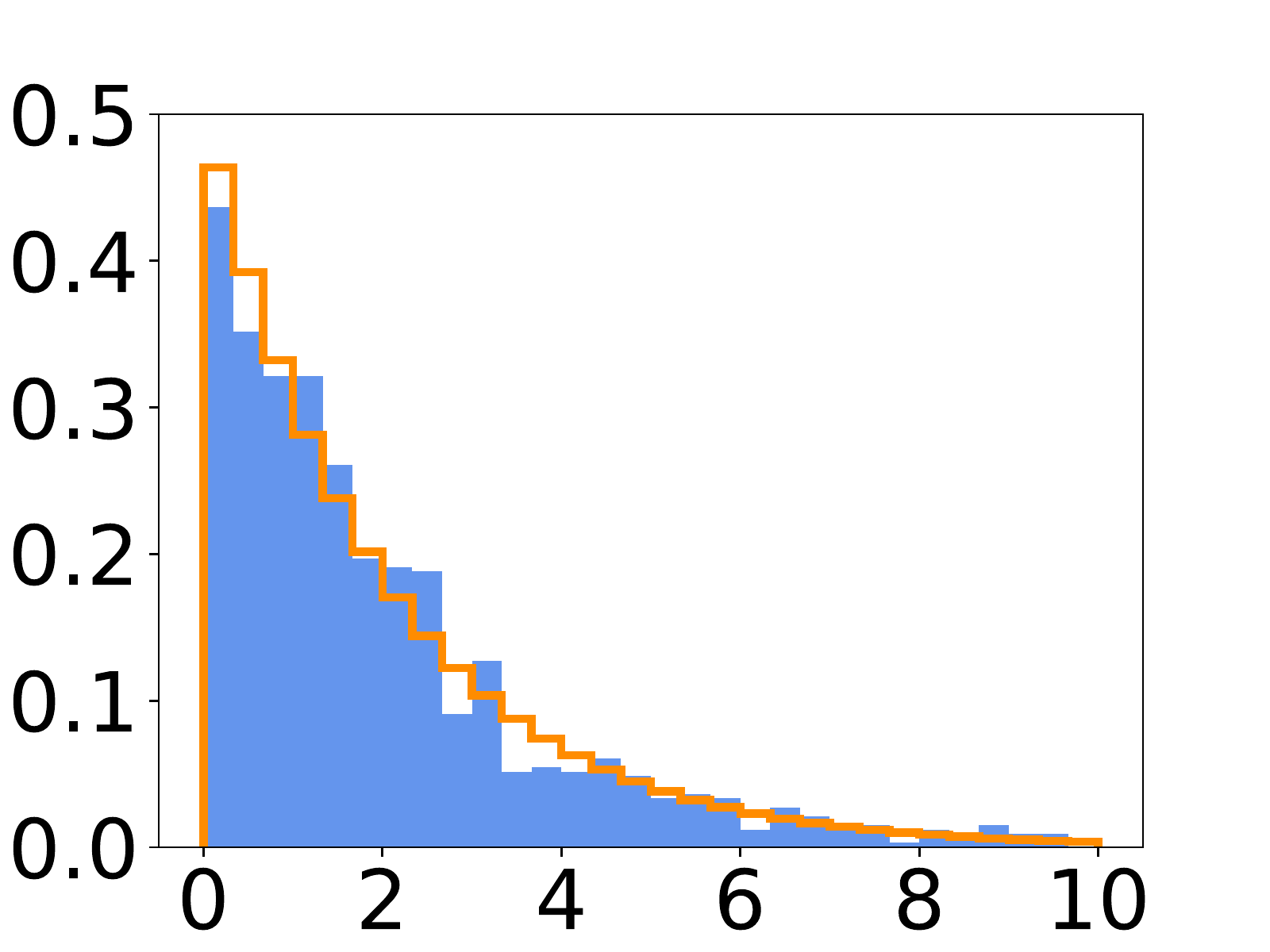}
}

    \subfigure[Ours(NN): $w_0$]{
        \includegraphics[width=0.3\textwidth]{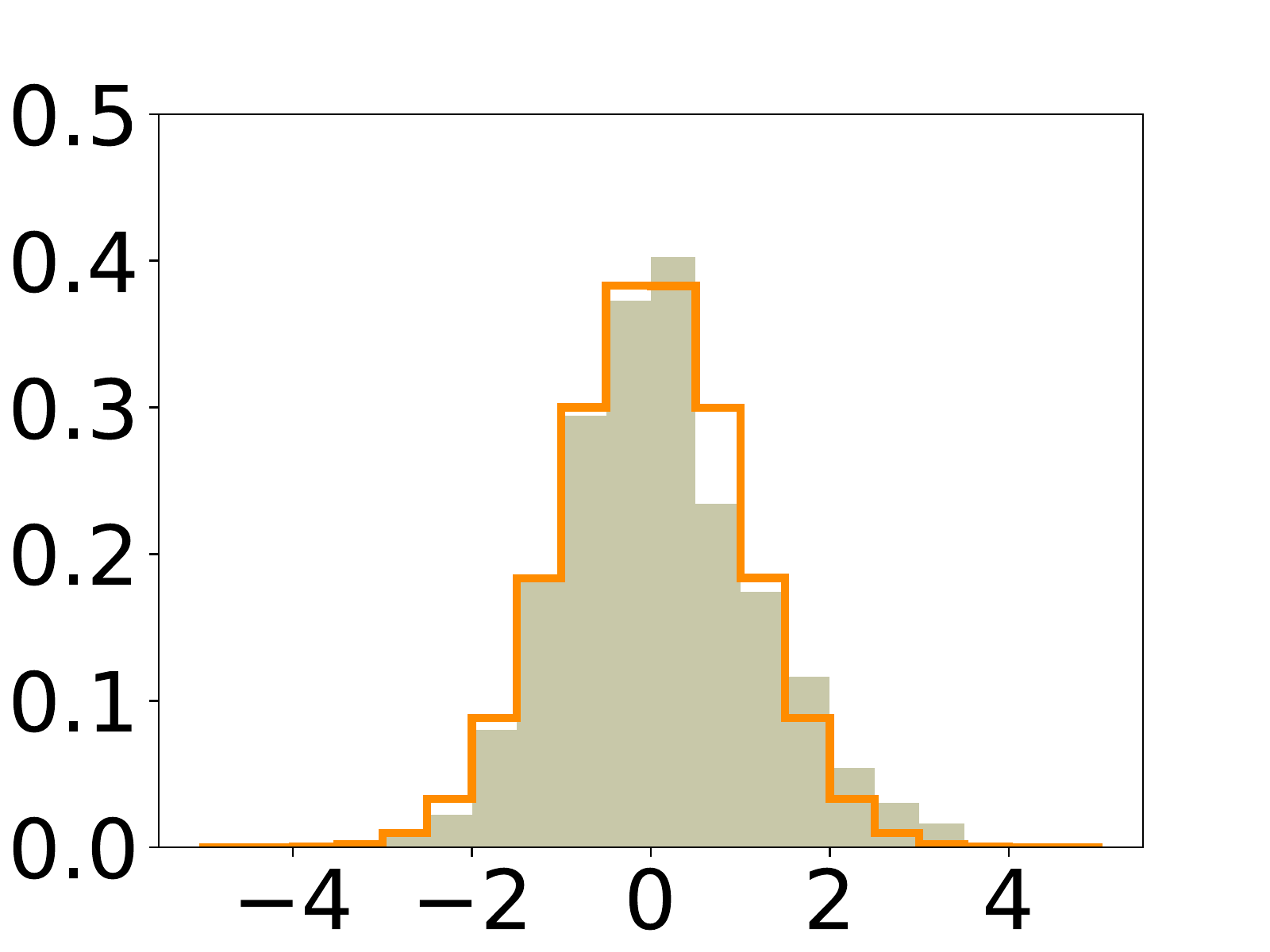}
    }
    \subfigure[Ours(NN): $w_1$]{
	\includegraphics[width=0.3\textwidth]{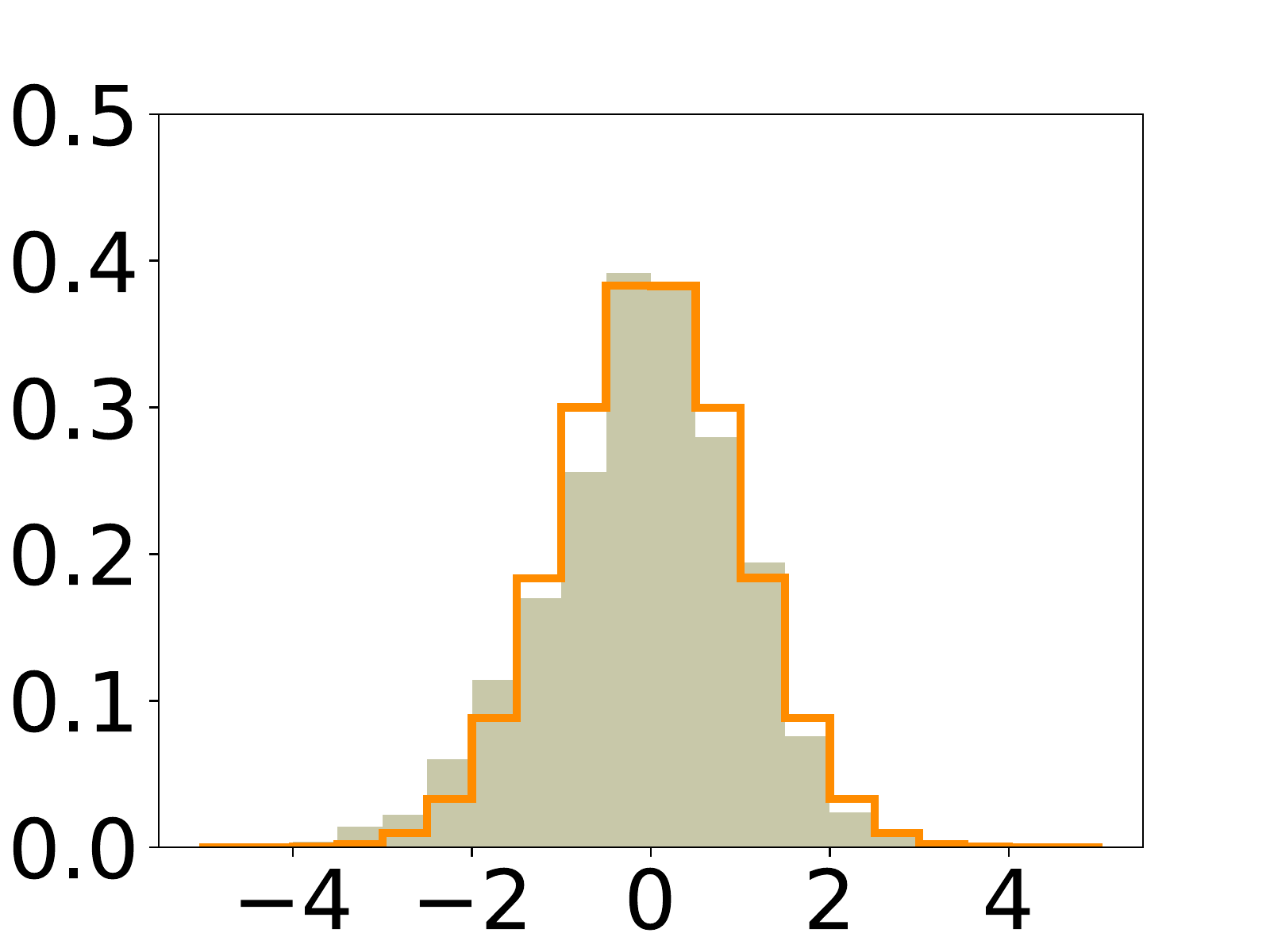}
}
    \subfigure[Ours(NN): $w$]{
	\includegraphics[width=0.3\textwidth]{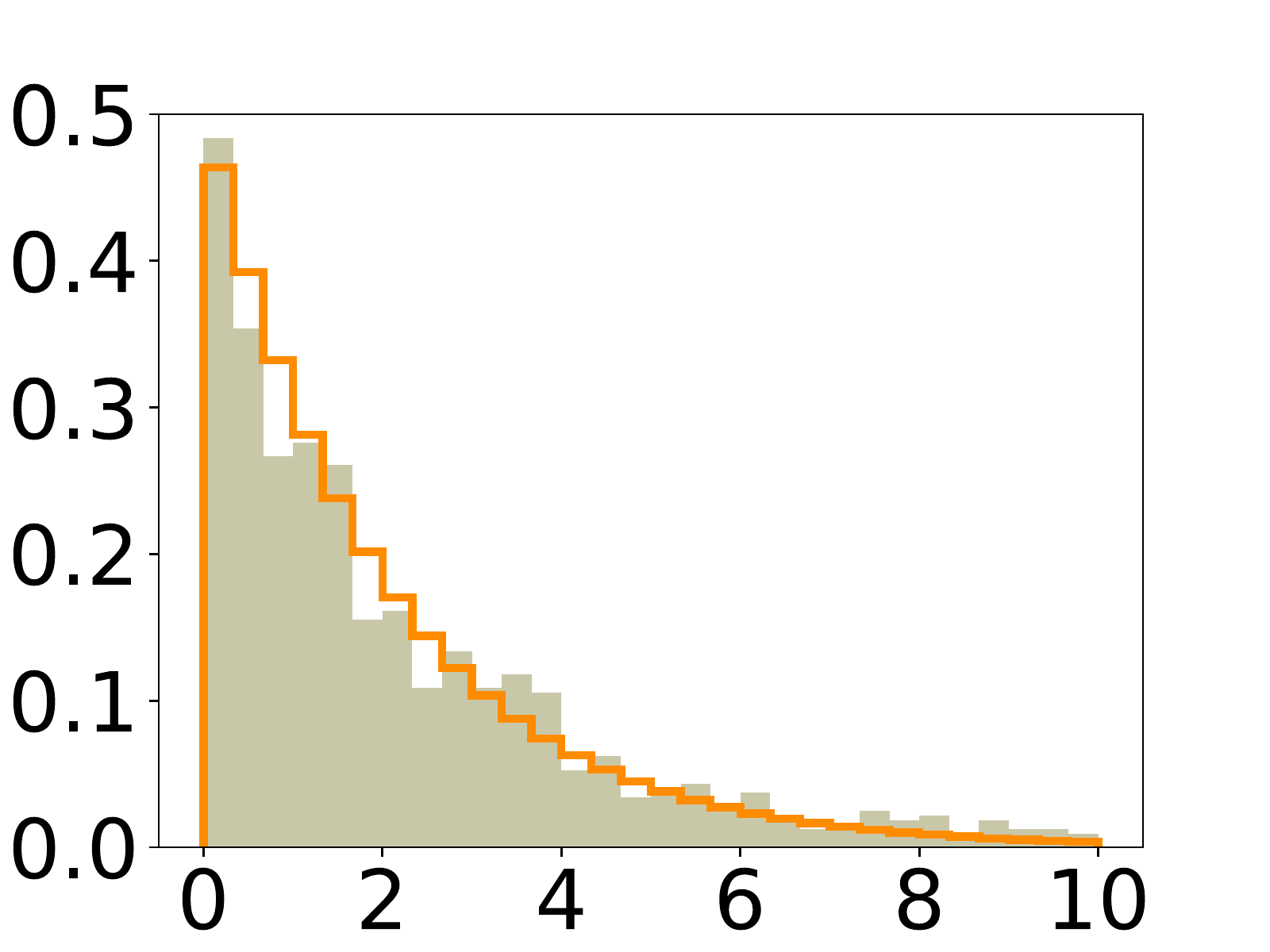}
}

    \subfigure[Ours(L): $w_0$]{
        \includegraphics[width=0.3\textwidth]{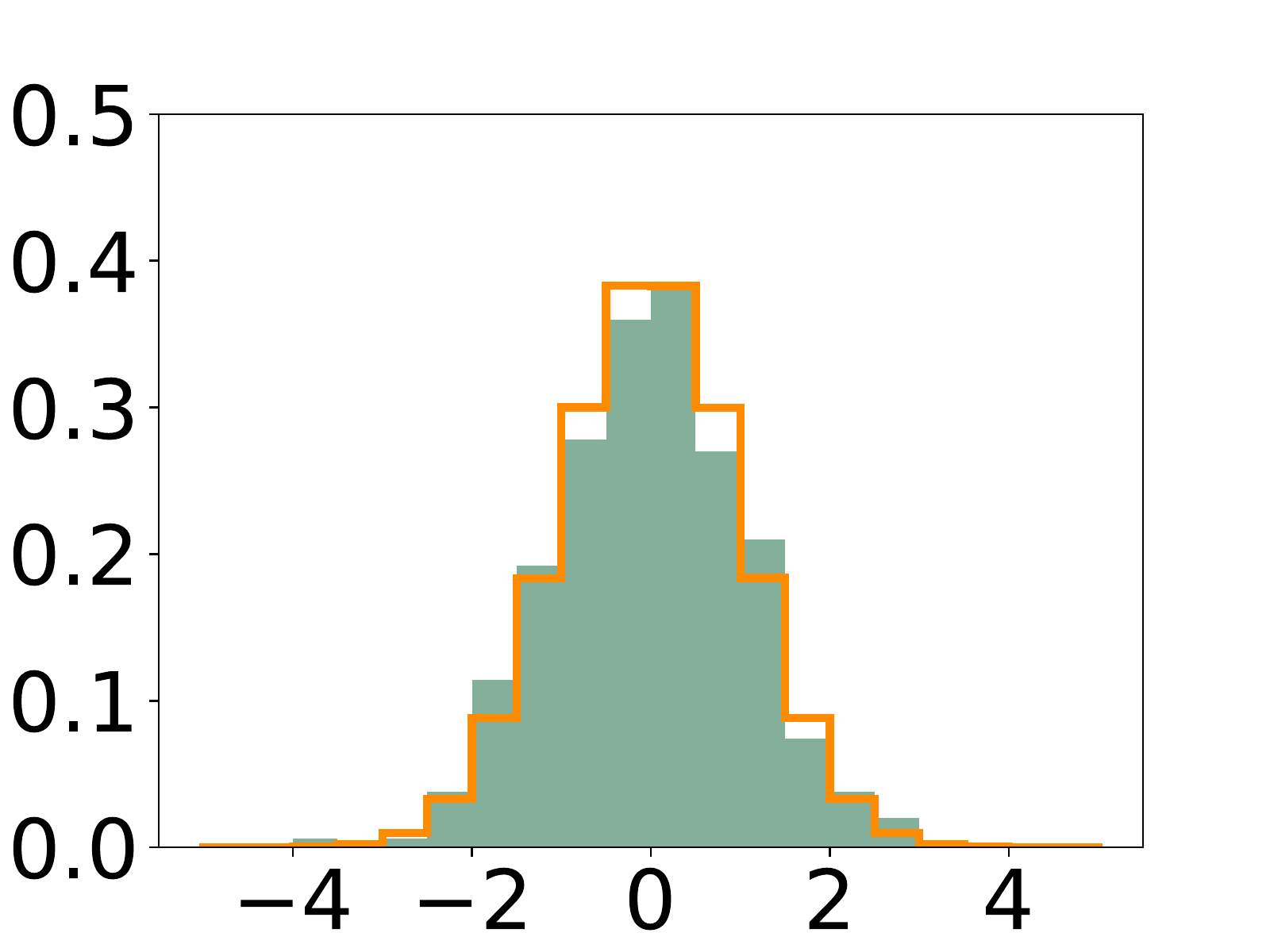}
    }
    \subfigure[Ours(L): $w_1$]{
	\includegraphics[width=0.3\textwidth]{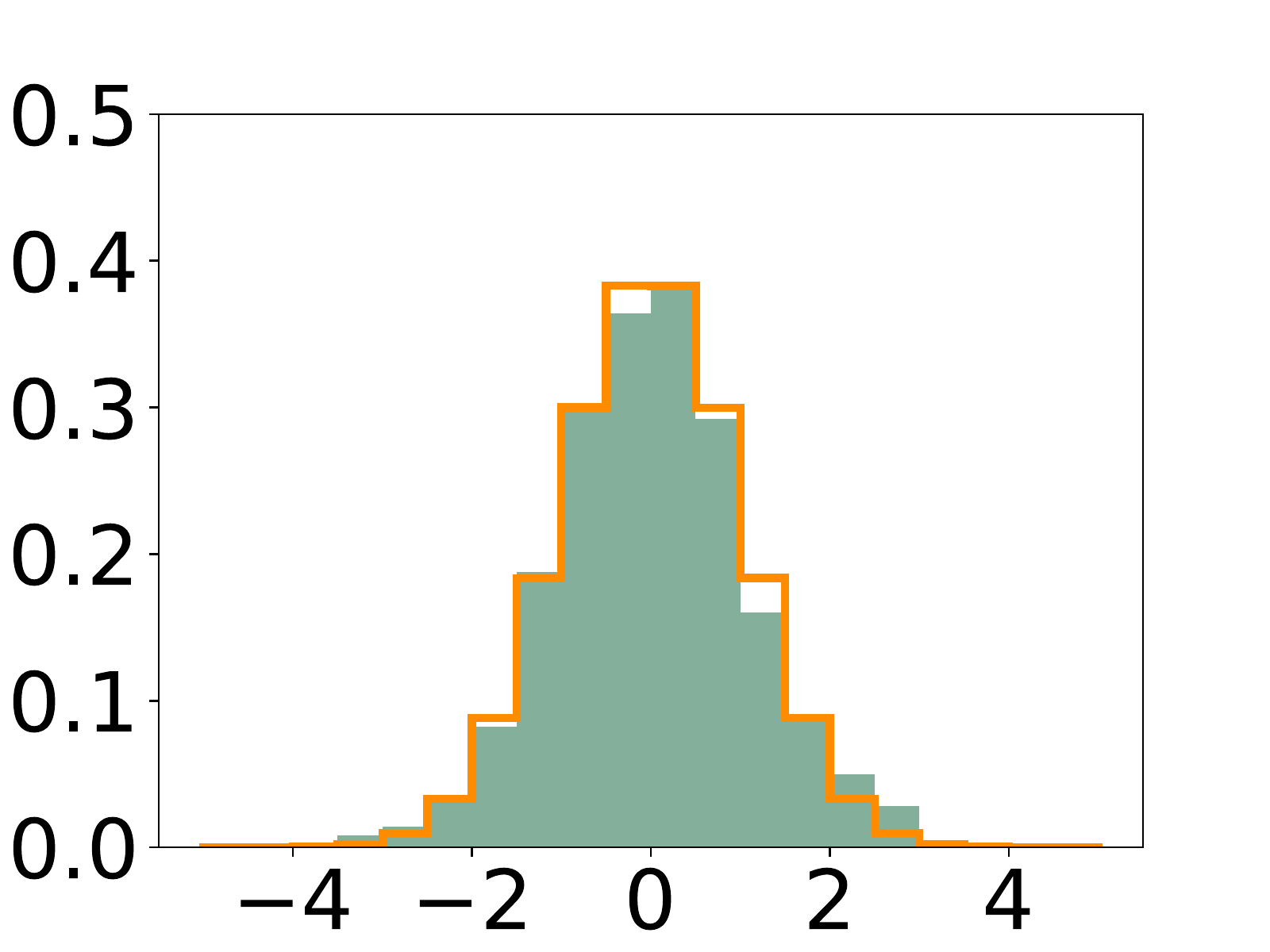}
}
    \subfigure[Ours(L): $w$]{
	\includegraphics[width=0.3\textwidth]{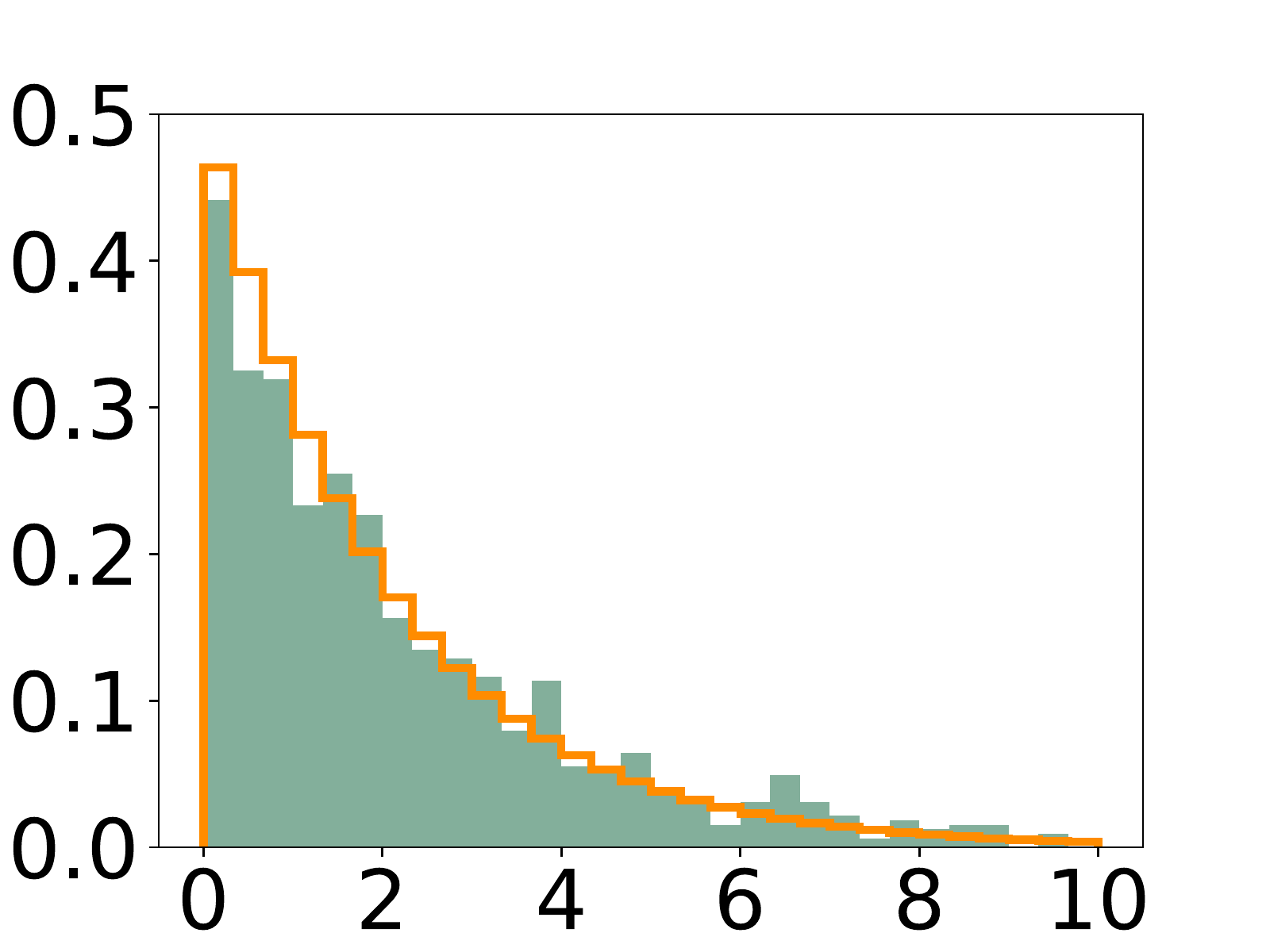}
}
    \caption{Test for parameters under linear scenarios. Orange Line shows the true distribution. Three methods can all reach theoretical distributions successfully.}
    \label{Fig: test_x}
\end{figure} 

\begin{table}[htbp]
	\centering 
	\caption{Correctness Comparison (Linear)}
	\label{Tab: LinearCorrectness} 
	\begin{tabular}{c|c|c|c}  
		\hline 
		\ & $\chi_2$ & normal ($w_0$)& normal ($w_1$) \\
		\hline
		Linear Reg & {\color{red} 0.0288($\pm 0.0059$)}& 0.0368($\pm 0.0050$)& 0.0367($\pm 0.0066$)\\
		Ours(NN) & 0.0643($\pm 0.0081$) & 0.0479($\pm 0.0041$)& 0.0427($\pm 0.0046$) \\		
		Ours(L) & 0.0474($\pm 0.0127$) & {\color{red} 0.0316($\pm 0.0067$)} & {\color{red} 0.0288($\pm 0.0062$)} \\
		\hline
    \end{tabular}
\end{table}

Moreover, Table~\ref{Tab: LinearEfficiency} shows the average standard deviation of each methods with their efficiency. 

\begin{table}[htbp]
	\centering 
	\caption{Efficiency Comparison (Linear)} 
	\label{Tab: LinearEfficiency} 
	\begin{tabular}{c|c|c}  
		\hline 
		\ & $w_0$ &  $w_1$  \\
		\hline
		Linear Reg & 0.0020 ($\pm 0.0000$)& 0.0035 ($\pm 0.0001$)\\
		Ours(NN) & 0.0196($\pm 0.0002$)  & 0.0323 ($\pm 0.0003$) \\		
		Ours(L) & 0.0037($\pm 0.0000$) & 0.0064 ($\pm 0.0000$)\\		
		\hline
	\end{tabular}
\end{table}

\section{Supplementary Experimental Results}
\label{Appendix: supplement}
For reducing the randomness in experiments, we conduct each experiment repeatedly six times.
We calculate the corresponding mean and average with respect to each trial, which is used to constructing confidence interval.

The $95\%$ confidence interval can be calculated as 
$$
\text{mean} + 1.96 * \text{std}
$$

We do not show the efficiency of linear regression in main text, since it returns fake significance level, and its efficiency is shown here just for completeness.
The full results are listed as follows. They corresponds to Table~\ref{Tab: Unbiaseness}, Table~\ref{Tab: NonLinearCorrectness}, Table~\ref{Tab: NonLinearEfficiency}, Table~\ref{Tab: LinearCorrectness}, Table~\ref{Tab: LinearEfficiency}. respectively.

\begin{table}[htbp]
	\centering 
	\caption{Bias Comparison (Non-Linear)}
	\begin{tabular}{c|c|c|c|c|c|c|c|c}  
		\hline 
		\ & Trial 1 & Trial 2 & Trial 3 & Trial 4 &Trial 5 &Trial 6 &average & std\\
		\hline
		Linear Reg ($w_0$)& -0.1718 & -0.1714 & -0.1695 & -0.1708 & -0.1705 & -0.1747 & -0.1715 & 0.0016\\
		Linear Reg ($w_1$)& 1.0033 & 1.0020 & 0.9976 & 1.0021 & 0.9989 & 1.0048 & 1.0015 & 0.0025\\
		Ours(NN) ($w_0$) & -0.1718 & -0.1644 & -0.1681 & -0.1746 & -0.1649& -0.1616 &-0.1676 & 0.0045\\
		Ours(NN) ($w_1$) &  1.0193 & 0.9911 & 1.004 & 1.0267 & 0.995 & 0.9924 & 1.0048 & 0.0137\\	
		Ours(L) ($w_0$) &  -0.1702 & -0.1673 & -0.1624 & -0.1667 & -0.1651& -0.169 & -0.1668 & 0.0025\\
		Ours(L) ($w_1$) &  1.0093 & 1.0018 & 0.9958 & 1.0055 & 0.9972 & 1.0029 & 1.0021 & 0.0046\\	
		\hline
    \end{tabular}
\end{table}

\begin{table}[htbp]
	\centering 
	\caption{Correctness Comparison (Non-Linear)}
	\begin{tabular}{c|c|c|c|c|c|c|c|c}  
		\hline 
		\ & Trial 1 & Trial 2 & Trial 3 & Trial 4 &Trial 5& Trial 6 &average & std\\
		\hline
		Linear Reg ($w$)& 0.1289 &  0.1119 & 0.0928 & 0.1176 & 0.1297& 0.1090   &0.1150 &0.0126\\
		Linear Reg ($w_0$)& 0.0669 &0.0602 & 0.0530& 0.0585 & 0.0648 & 0.0774 & 0.0635 &0.0077 \\
		Linear Reg ($w_1$)&  0.0727& 0.0657 & 0.0688 & 0.0583 & 0.0782 & 0.0853 & 0.0715 &0.0087\\
		Ours(NN) ($w$) &  0.0759 & 0.0763 & 0.0649 & 0.0570& 0.0687& 0.0648 & 0.0679 &0.0067\\
		Ours(NN) ($w_0$) & 0.0376 &0.0343 & 0.0217& 0.0266 & 0.0440 & 0.0314 & 0.0326 &0.0072\\
		Ours(NN) ($w_1$) &  0.0702& 0.0645 & 0.0580 & 0.0555 & 0.0756& 0.0661 & 0.0650&0.0068\\		
		Ours(L) ($w$) &  0.0919& 0.0907 & 0.0688 & 0.0729 & 0.0801 & 0.0814 & 0.0810 &0.0084\\
		Ours(L) ($w_0$) &  0.0461& 0.0683 & 0.0394 & 0.0411 &0.0526 & 0.0459 & 0.0489&0.0096\\
		Ours(L) ($w_1$) &  0.0264 & 0.0395 & 0.0228 & 0.0273 & 0.0271 & 0.0224 & 0.0276&0.0057\\	
		\hline
    \end{tabular}
\end{table}


\begin{table}[htbp]
	\centering 
	\caption{Efficiency Comparison (Non-Linear)}
	\begin{tabular}{c|c|c|c|c|c|c|c|c}  
		\hline 
		\ & Trial 1 & Trial 2 & Trial 3 & Trial 4 &Trial 5 &Trial 6 &average & std\\
		\hline
		Linear Reg ($w_0$)& 0.0148 & 0.0148 & 0.0148 & 0.0148 & 0.0148 & 0.0147 & 0.0148 & 0.0000\\
		Linear Reg ($w_1$)& 0.0257 & 0.0256 & 0.0256 & 0.0256 & 0.0256 & 0.0256 & 0.0256 & 0.0000\\
		Ours(NN) ($w_0$) & 0.0219 & 0.0212 & 0.0210 & 0.0216 & 0.0211& 0.0213 & 0.0214 & 0.0003\\
		Ours(NN) ($w_1$) &  0.0594 & 0.0574 & 0.0570 & 0.0590 & 0.0569 & 0.0661 & 0.0593 & 0.0032\\	
		Ours(L) ($w_0$) &  0.0333 & 0.0325 & 0.0326 & 0.0331 & 0.0329& 0.0324 & 0.0328 & 0.0003\\
		Ours(L) ($w_1$) &  0.0610 & 0.0600 & 0.0600 & 0.0609 & 0.0604 & 0.0600 & 0.0604 & 0.0004\\	
		\hline
    \end{tabular}
\end{table}

\begin{table}[htbp]
	\centering 
	\caption{Correctness Comparison (Linear)}
	\begin{tabular}{c|c|c|c|c|c|c|c|c}  
		\hline 
		\ & Trial 1 & Trial 2 & Trial 3 & Trial 4 &Trial 5  &Trial 6 & average &std\\
		\hline
		Linear Reg ($w$)& 0.0234 & 0.0339 & 0.0244 & 0.0232 & 0.0248 & 0.0433 & 0.0288 &0.0074\\
		Linear Reg ($w_0$)& 0.0361 & 0.0395 & 0.0380 &0.0478 &0.0301 & 0.0294&0.0368 & 0.0062\\
		Linear Reg ($w_1$)&  0.0429 & 0.0341 & 0.0384& 0.0501&0.0278 & 0.0270 & 0.0367 & 0.0082\\
		Ours(NN) ($w$) &  0.0708 & 0.0746 & 0.0496 & 0.0727& 0.0512 & 0.0666 & 0.0643 & 0.0101\\
		Ours(NN) ($w_0$) & 0.0511 & 0.0539 & 0.0519&0.0431 & 0.0479 & 0.0395 & 0.0479 & 0.0051\\
		Ours(NN) ($w_1$) &  0.0487 & 0.0412 & 0.0426&0.0313& 0.0443 & 0.0479 & 0.0427 & 0.0057\\
		Ours(L) ($w$) &  0.0435 & 0.0618 & 0.0211 & 0.0657 & 0.0347 & 0.0576& 0.0474 & 0.0159\\
		Ours(L) ($w_0$) &  0.0314 & 0.0250 & 0.0218 &0.0336&0.0481 & 0.0295 & 0.0316 & 0.0084\\
		Ours(L) ($w_1$) &  0.0227 & 0.0212 & 0.0257 & 0.0263&0.0328 & 0.0441& 0.0288 & 0.0078\\
		\hline
    \end{tabular}
\end{table}

\begin{table}[htbp]
	\centering 
	\caption{Efficiency Comparison (Linear)}
	\begin{tabular}{c|c|c|c|c|c|c|c|c}  
		\hline 
		\ & Trial 1 & Trial 2 & Trial 3 & Trial 4 &Trial 5&Trial 6 & average & std\\
		\hline
		Linear Reg ($w_0$)& 0.0020& 0.0020 & 0.0020 & 0.0020 & 0.0020 & 0.0020&0.0020 & 0.000\\
		Linear Reg ($w_1$)& 0.0034 & 0.0035 & 0.0035 & 0.0035 & 0.0034 & 0.0034&0.0035 & 0.0001\\
		Ours(NN) ($w_0$) & 0.0138& 0.0135 & 0.0136 & 0.0133 & 0.0139 & 0.0138&0.0137 & 0.0002\\
		Ours(NN) ($w_1$) & 0.0230& 0.0225 & 0.0227&0.0222 & 0.0234 & 0.0228&0.0228 & 0.0004\\	
		Ours(L) ($w_0$) & 0.0036& 0.0037 & 0.0036 & 0.0036& 0.0036 & 0.0037&0.0036 & 0.0000\\
		Ours(L) ($w_1$) & 0.0063& 0.0063 & 0.0064& 0.0063&  0.0063 & 0.0063&0.0063 & 0.0000\\		
		\hline
    \end{tabular}
\end{table}


\end{document}